\documentclass{svproc}
\usepackage{float}
\usepackage{url}
\usepackage{xcolor,overpic,contour}
\usepackage{graphicx,enumerate}
\usepackage{amsmath,amssymb}
\usepackage{longtable}
\usepackage{array}
\newcolumntype{P}[1]{>{\raggedright\arraybackslash}p{#1}} 
\usepackage{graphicx} 
\def\CC{{\mathbb C}}
\def\RR{{\mathbb R}}
\def\Vkt#1{{\mathbf #1}}

\def\dach#1{\widehat{#1}}
\newcommand{\mVkt}[1]{\dach{\Vkt #1}}  
\newcommand{\bVkt}[1]{\bar{\Vkt #1}} 
\contourlength{0.2mm}
\newtheorem{thm}{Theorem}

\begin{document}
\mainmatter

\title{Higher-Order Flexible Configurations of Planar Parallel Manipulators Constructed by Averaging}

\titlerunning{Higher-Order Flexible Configurations Constructed by Averaging}

\author {Yudi Zhao \and Georg Nawratil}

\authorrunning{ Y. Zhao and G. Nawratil} 

\institute{Institute of Discrete Mathematics and Geometry \& \\ 
Center for Geometry and Computational Design, TU Wien, Austria\\
\email{\{yzhao,nawratil\}@geometrie.tuwien.ac.at}\\ 
WWW home page: \texttt{https://www.geometrie.tuwien.ac.at/zhao/}\\
\texttt{https://www.geometrie.tuwien.ac.at/nawratil/}}

\maketitle 

\begin{abstract}
This paper investigates singular configurations of  planar $3$-RPR parallel manipulators, which result from applying the averaging technique to solution pairs of their direct kinematic problem. Without computing the zeros of the corresponding degree 6 polynomial we parametrize the input pairs and 
determine their relative orientation in a way that the flexion order of the averaged configurations increases. Moreover, the obtained results are visualized for concrete examples. The presented methodology can also be used for studying the spherical and spatial analogues of planar $3$-RPR parallel manipulators.

\end{abstract}

\section{Introduction}

A bar-joint framework $\mathcal{G}(\mathcal{X})$ consists of a set $\mathcal{X}=\left\{X_{1}, X_2,\ldots \right\}$ of distinct vertices
and a graph $\mathcal{G}$  on $\mathcal{X}$ encoding the combinatorial structure. 
A vertex $X_i$ corresponds to 
a rotational/spherical joint (without clearance) in the case of a planar/spatial framework and an edge connecting two vertices corresponds to a bar between the associated joints.

By defining the lengths of the bars, which are assumed to be non-zero, the intrinsic metric of the framework is fixed. 
In general this assignment does not uniquely determine the embedding of the framework into the Euclidean space $\RR^s$($s=2$ for planar and $s=3$ for spatial frameworks), thus such a framework can have 
different incongruent realisations. Note that a realisation is denoted by $\mathcal{G}(\Vkt X)$ with $\Vkt X=(\Vkt x_1, \Vkt x_2, \ldots )$ where $\Vkt x_i$ denotes the corresponding vector of the vertex $X_i$. 

According to \cite{nawratil:global} the number of coinciding realisations can be used to define the flexion order of a realisation as follows:

\begin{definition}[Flexion order of a realisation]\label{def:order}
If a realisation $\mathcal{G}(\Vkt X)$ 
does not belong to a continuous flexion of the framework (over $\CC$), then its flexion order is defined as $c - 1$, where $c$ is the
multiplicity of $\mathcal{G}(\Vkt X)$, i.e., the number of coinciding framework realisations at $\mathcal{G}(\Vkt X)$.
\end{definition}

\subsection{Averaging Method}

According to Whiteley~\cite{whiteley:1997} the methodology of averaging, which is illustrated in Fig.\ \ref{fig:avg3rpr}a, can be used to synthesise an infinitesimally flexible configuration  from two incongruent\footnote{Non-existence of a direct or an indirect isometry $\iota$: $\RR^s\rightarrow \RR^s$  with $\iota(\Vkt x_k)=\Vkt x'_k$ $\forall k$.} realisations. This construction, which was already known to Pogorelov according to \cite[Sec.\ 4.4]{izmestiev:2009}, is described next:

\begin{definition}[Averaged configuration] \label{def:averconf}
Take the midpoints $\bVkt x_k:=\tfrac{\Vkt x_k + \Vkt x'_k}{2}$ of corresponding vertices $\Vkt x_k$  and $\Vkt x'_k$  of 
two incongruent framework realisations $\mathcal{G}(\Vkt X)$ and $\mathcal{G}(\Vkt X')$, respectively. In case these points define a framework of the same combinatorial structure\footnote{For bar-joint frameworks this is not always the case as two edge-connected vertices can coincide after averaging, which results in a bar of length zero.}, we call $\mathcal{G}(\bVkt X)$ the averaged configuration.  
\end{definition}

For this averaged configuration  $\mathcal{G}(\bVkt X)$ the following statement holds true:

\begin{thm}[Averaging theorem for bar-joint frameworks] 
\label{thm:aver}
Assume that  $\mathcal{G}(\bVkt X)$  denotes the averaged configuration of the incongruent realisations 
$\mathcal{G}(\Vkt X)$ and $\mathcal{G}(\Vkt X')$ of a bar-joint framework. Then $\mathcal{G}(\bVkt X)$ is at least of flexion order 1.
\end{thm}

A simple proof for this theorem is e.g.\ given in \cite{stachel:iwssip,stachel:saj}.
This averaging technique has some degrees of freedom as the relative pose of the two incongruent realisations can be chosen arbitrarily. It can easily be seen that the translational part does not affect the shape of the averaged configuration $\mathcal{G}(\bVkt X)$. 
\begin{lemma}[Translation invariance of the averaged configuration] \label{lem:invariant}
$\mathcal{G}(\bVkt X)$ remains invariant under the replacement of $\mathcal{G}(\Vkt X)$ and $\mathcal{G}(\Vkt X')$ by $\mathcal{G}(\mVkt X)$ and $\mathcal{G}(\mVkt X')$, respectively, with 
$\mVkt X=(\Vkt x_1+\Vkt t, \Vkt x_2+\Vkt t,\ldots)$ and 
$\mVkt X'=(\Vkt x'_1-\Vkt t, \Vkt x'_2-\Vkt t,\ldots)$ where $\Vkt t\in\mathbb R^s$ denotes the translation vector. Therefore, the  flexion order of $\mathcal{G}(\bVkt X)$ is invariant under a change of the relative position of $\mathcal{G}(\Vkt X)$ and $\mathcal{G}(\Vkt X')$.
\end{lemma}

However, the shape of the averaged configuration $\mathcal{G}(\bVkt X)$ is not invariant under changes in the relative orientation of the two incongruent realisations. We therefore seek for a geometric characterisation of orientations for which the averaged configuration $\mathcal{G}(\bVkt X)$ exhibits a higher flexion order. \medskip

In this paper, which is organised as follows, we focus on planar parallel manipulators. 
In Section \ref{sec:ppm} we recall the geometric setting and notation of these manipulators and their interpretation as bar-plate frameworks, which allows 
an efficient computation of configurations with a higher flexion order using Blaschke-Gr\"unwald parameters (see Sections \ref{sec:direct} and \ref{sec:order12}). 
In Subsection \ref{sec:cases} we distinguish different kinds of pairing incongruent realisations and in Section \ref{sec:parametrisation} we parametrize the corresponding sets of these pairs. 
Based on this preparatory work we can formulate our theorems in a precise manner in Section \ref{sec:results}, where we also sketch the proofs and illustrate the obtained results. 
Finally, in Section \ref{sec:conc} the paper is concluded with some final remarks and an outlook to future research. 
\begin{figure}[t]
\begin{center}
\hfill
\begin{overpic}[height=25mm]{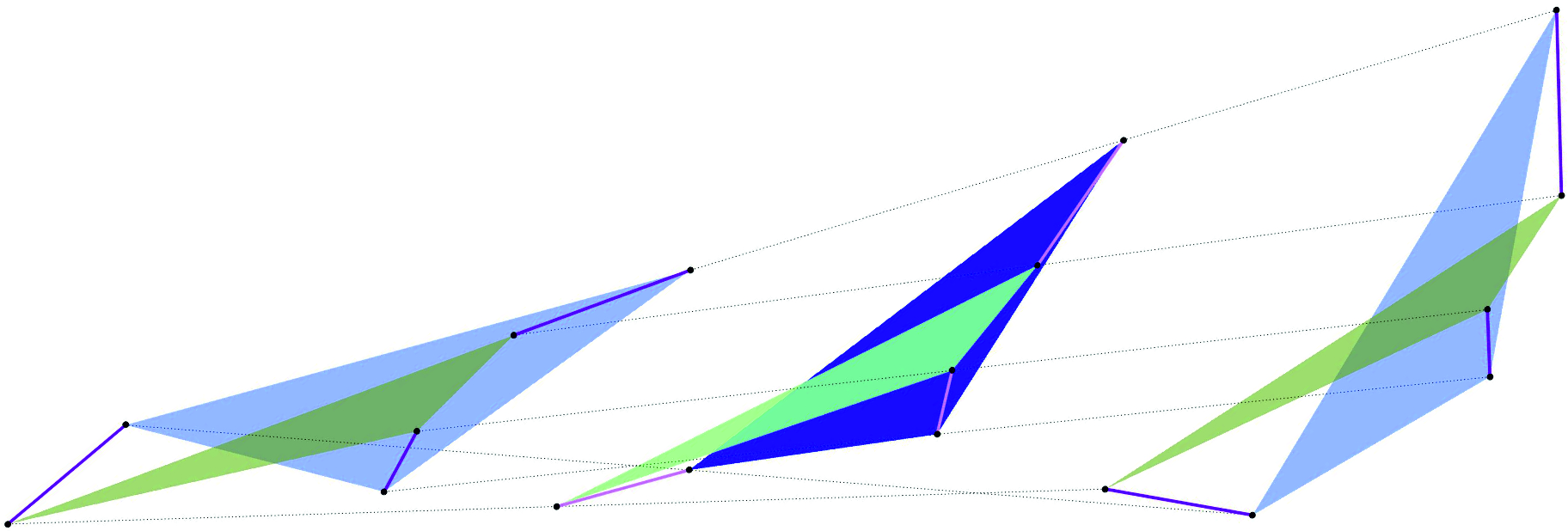}
   \begin{scriptsize}
   \put(-5,0){a)} 
\put(41,7){$\bVkt x_3$}   
\put(70,26){$\bVkt x_2$}
\put(58,3){$\bVkt x_1$}
\put(32,3){$\bVkt x_6$}
\put(67,15){$\bVkt x_5$}
\put(63,9){$\bVkt x_4$}
\end{scriptsize}
\end{overpic}
\hfill
  \begin{overpic}
   [height=25mm]{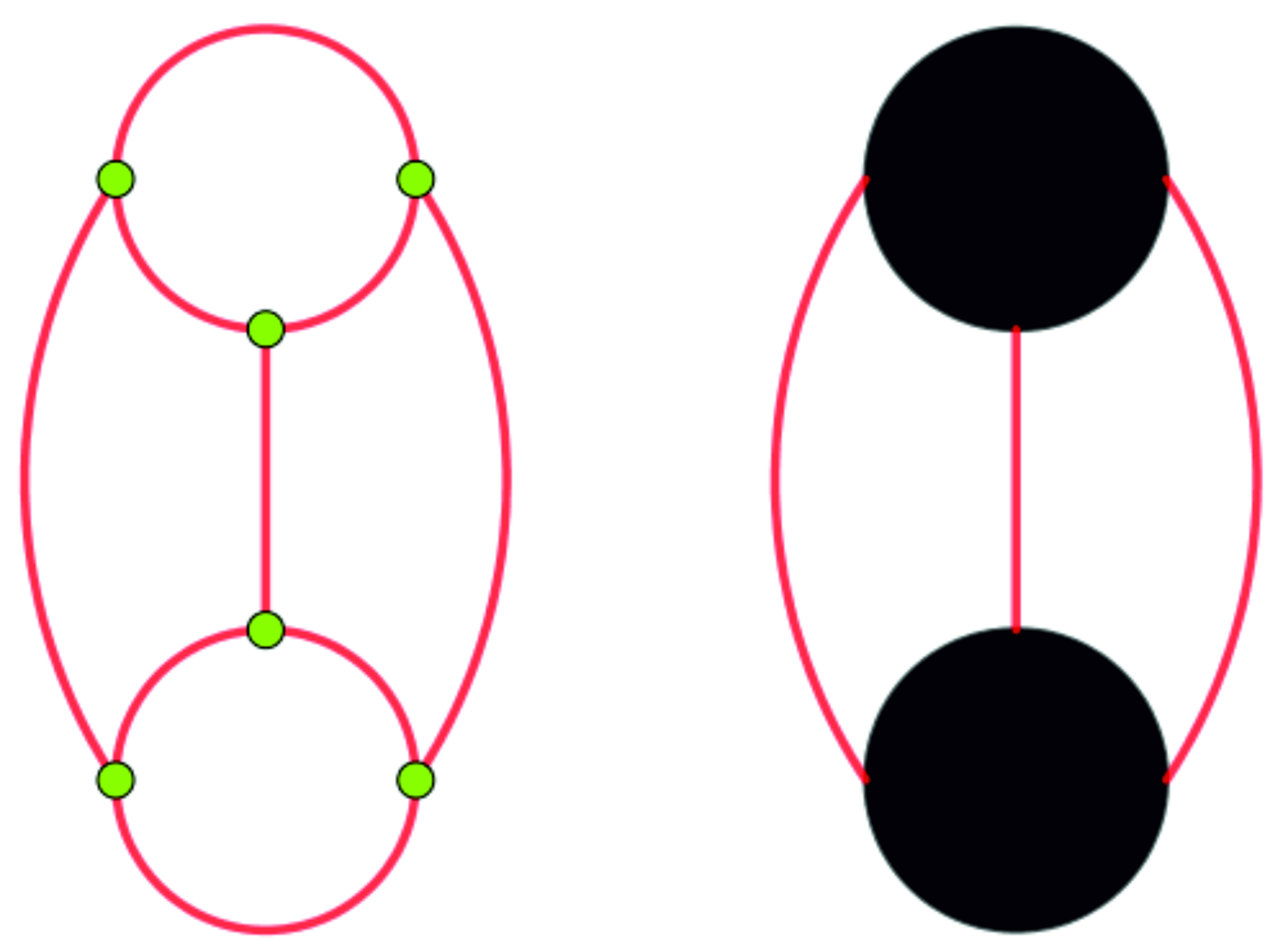}
   \begin{scriptsize}
\put(0,0){b)}
\put(54,0){c)}
\end{scriptsize}
\end{overpic} 
\end{center}	
\caption{
(a) Averaging method applied to a planar $3$-RPR 
parallel manipulator. The averaged configuration $\mathcal{G}(\bVkt X)$ is even of flexion order 2 and corresponds to Case ($I_i$) of Theorem \ref{thm:setA}. 
The underlying graphs of this parallel mechanism interpreted as bar-joint framework (b) and pin-jointed bar-plate framework (c), respectively. In the graphs the bars are printed in red, vertices in green and plates in black.}
  \label{fig:avg3rpr}
\end{figure}
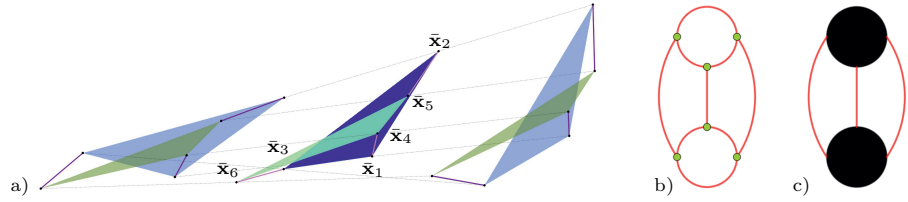    
\section{Planar Parallel Manipulator}\label{sec:ppm}

A planar $3$-RPR parallel manipulator consists of a triangular base and triangular platform, both embedded in the Euclidean plane, which are connected by three pairwise distinct legs
arranged in a revolute--prismatic--revolute (RPR) joint sequence. The revolute joints of a leg are attached to the base anchor point $X_i$ and the corresponding platform anchor point $X_{i+3}$ for $i=1,2,3$, respectively, while the prismatic joint in each leg is actuated and controls the leg length. As a consequence, the moving platform possesses three degrees of freedom (DoFs) in the plane, namely two translational and one rotational DoF. We assume that for the non-zero lengths $r_i$ of the three legs fixed values are given, which is equivalent to the locking of the prismatic joints. 

One can interpret the triangular base and platform either as (i) triangular bar structures or as (ii) triangular plates. In Case (i) the planar $3$-RPR manipulator with locked P joints can be seen as a classical bar-joint framework and in Case (ii) as a pin-jointed bar-plate framework. According to \cite[Sec.\ 4]{nawratil:global} these different interpretations also effect the discussion of the flexion order of a configuration, in case that the triangular platform and/or base degenerate/s into a line.
In the remainder of the paper we interpret the planar $3$-RPR  parallel manipulator with locked P joints as a pin-jointed bar-plate framework. 

{\bf The advantage of this interpretation} is that configurations with a higher flexion order can be computed in a very efficient way using Blaschke-Gr\"unwald parameters (see Sections \ref{sec:direct} and \ref{sec:order12}). Moreover, this interpretation also allows that two anchor points of the platform/base coincide without a change of the underlying graph (see Fig.\ \ref{fig:avg3rpr}bc).
We only exclude the degenerated cases
that neither the three  base anchor points nor the three platform anchor points collapse to a single point. 
This non-degeneracy assumption ensures that the geometric constraints defining the planar parallel manipulator are well posed and that the kinematic equations describe a genuine planar parallel manipulator.

{\bf The disadvantage of this interpretation} is that the averaging theorem does not come for free (see the later given Theorem \ref{thm:setB}) as it is originally stated for bar-joint frameworks (cf.\ Theorem \ref{thm:aver}). 
However, Definition \ref{def:averconf} and Lemma \ref{lem:invariant} are formulated in a way that they are also valid for this interpretation. In context of  Definition \ref{def:averconf} it should be noted that the combinatorial structure of a pin-jointed bar-plate framework can only be violated  by the averaging construction if
either the length of a leg gets zero
or  two legs coincide (i.e.\ same platform  and same base anchor points). 

\subsection{Algebraic Formulation of the Direct Kinematic Problem}\label{sec:direct}

Let $\Vkt x_j=(a_j,b_j)^{\mathrm T}$ for $j=4,5,6$ denote the coordinates of a point $X_j$ on the moving platform with respect to the moving frame. The pose of the platform relative to the fixed frame is described using the Blaschke-Gr\"unwald parameters $(q_0 : q_1 : q_2 : q_3)$, which may be regarded as homogeneous coordinates in the projective space $\mathbb{P}^3$. It is well known that, after excluding the line $q_0=q_1=0$, this parametrisation establishes a bijection with the planar Euclidean motion group $\mathrm{SE}(2)$. 
The transformation of the coordinates of a point $X_j$ from the moving frame to the fixed frame is given by 
\begin{equation}
\begin{pmatrix}
a_j \\ b_j
\end{pmatrix}
\mapsto
\frac{1}{q_0^2 + q_1^2}
\left(
\begin{array}{c@{\hspace{8pt}}c}
q_0^2 - q_1^2 &-2q_0 q_1 \\
2q_0 q_1 & q_0^2 - q_1^2
\end{array}
\right)
\begin{pmatrix}
a_j \\ b_j
\end{pmatrix}
+
\frac{1}{q_0^2 + q_1^2}
\begin{pmatrix}
2q_1 q_2 + 2q_0 q_3 \\
2q_1 q_3 - 2q_0 q_2
\end{pmatrix}.
\label{eq:bg_transform}
\end{equation}
Under consideration of the normalising 
condition 
\begin{equation}
c_0 := q_0^2 + q_1^2 - 1 = 0
\label{eq:normalisation}
\end{equation}
the geometric condition that the platform point $X_j$ lies on a circle of radius $r_i$ centred at the base point $X_i$ with coordinates $\Vkt x_i=(a_i,b_i)^{\mathrm T}$ for $i=j-3$  
with respect to the fixed frame can be expressed by a quadratic constraint $c_i = 0$ with
\begin{align}
c_i :=\;&
2a_i a_j q_1^2 - 2a_i a_j q_0^2 + 4a_i b_j q_0 q_1
- 4b_i a_j q_0 q_1 - 2b_i b_j q_0^2 + 2b_i b_j q_1^2 \notag\\
&+ (a_j^2+b_j^2)(q_0^2+q_1^2)
- 4a_i(q_0 q_3 + q_1 q_2) + 4b_i(q_0 q_2 - q_1 q_3) \notag\\
&+ 4a_j(q_0 q_3 - q_1 q_2) - 4b_j(q_0 q_2 + q_1 q_3)
+ a_i^2 + b_i^2 + 4(q_2^2+q_3^2) - r_i^2 
\label{eq:sphere_constraint}
\end{align}
according to \cite{husty:1994}. 
Besides the geometry of base plate and the platform plate,   the intrinsic metric of the bar-plate framework is completed by the leg lengths $r_i$. Consequently, the realisations $\mathcal{G}(\mathbf{X})$ correspond to the solutions of the system $c_0 = c_1 = c_2 = c_3 = 0$.
It is known that this so-called direct kinematic problem admits at most six solutions. 
Thus configurations with a flexion order of at most 5 are possible (see \cite{husty:2023,jank,stachel:planar,wohlhart}). Note that flexion order 6 would already imply that the configuration has to belong to a continuous motion known as self-motion.

\subsection{Characterisation of Configurations with Flexion Order 1 and 2}\label{sec:order12}

It is well known that configurations with a flexion order of at least 1 can be characterized by the geometric condition that the carrier lines of the three legs belong to a pencil of lines; i.e. they have a common point which can also be an ideal point.
The corresponding algebraic condition can be obtained from the rigidity matrix, which is given by $R_{\mathcal{G}(\mathbf{X})} = (\nabla c_0, \nabla c_1, \nabla c_2, \nabla c_3)$ according to \cite{husty:1994}, 
where the operator $\nabla$ denotes the gradient composed of the partial derivatives with respect to $q_0,\ldots, q_3$.
The variety $V_1$ of configurations with a flexion order of at least 1 is defined as the zero set of 
$$
s := \det\bigl(R_{\mathcal{G}(\mathbf{X})}\bigr).
$$
For the computation of the variety $V_2$  of configurations with a flexion order of at least 2 we 
use the following approach presented in \cite{nawratil:global}. 
The condition $s=0$  also implies that the tangent hyperplanes of the constraint hypersurfaces $c_0,\dots,c_3$ in $\mathbb{P}^3$ intersect in a common line. If this line is again tangent to $V_1$ then the corresponding configuration has at least flexion order 2. 
This is equivalent to the condition $s_0 = s_1 = s_2 = s_3 = 0$
with
$$
\begin{aligned}
s_0 &:= \det(\nabla c_1, \nabla c_2, \nabla c_3, \nabla s), 
&\qquad
s_1 &:= \det(\nabla c_0, \nabla c_2, \nabla c_3, \nabla s), \\[0.5ex]
s_2 &:= \det(\nabla c_0, \nabla c_1, \nabla c_3, \nabla s), 
&\qquad
s_3 &:= \det(\nabla c_0, \nabla c_1, \nabla c_2, \nabla s),
\end{aligned}
$$  
We introduce the second-order ideal
\begin{equation}
\mathcal{I}_2 = \langle s, s_0, s_1, s_2, s_3 \rangle,
\label{eq:ideal}
\end{equation}
whose variety $V_2$ determines configurations with a flexion order of at least 2.

Note that $V_2$ also contains singular points of $V_1$ as they are characterized by the condition $s=\nabla s=0$. 
According to \cite[Theorem 3]{adit}, the variety $V_1$ possesses singularities only for certain special architectural designs, in addition to those arising from the chosen parametrisation, which correspond to the line $q_0 = q_1 = 0$. The only configuration which corresponds to a singularity of $V_1$ without having a leg of zero length is where all six anchor points are collinear.

\subsection{Pairing of Incongruent Realisations}\label{sec:cases}

In general there exists 24 incongruent realisations (over $\CC$) of a non-degenerated pin-jointed bar-plate framework. The pairing of two arbitrary realisations
 $\mathcal{G}(\Vkt X)$ and $\mathcal{G}(\Vkt X')$
out of these 24 possibilities has to fall into  one of the following four pairwise disjunct sets:

\begin{enumerate}[{{Set} A)}]
    \item 
    There exist direct isometries $\alpha$ and $\beta$ of $\RR^2$ with $\alpha(\Vkt x_i)=\Vkt x'_i$ for $i=1,2,3$ and $\beta(\Vkt x_j)=\Vkt x'_j$ for $j=i+3$. 
    \item 
    There exist indirect isometries $\alpha$ and $\beta$ of $\RR^2$ with $\alpha(\Vkt x_i)=\Vkt x'_i$ for $i=1,2,3$ and $\beta(\Vkt x_j)=\Vkt x'_j$ for $j=i+3$. Moreover, neither the platform points nor the base points are collinear. 
     \item 
    There exists a direct isometry $\alpha$ of $\RR^2$  and an indirect isometry $\beta$ of $\RR^2$ with $\alpha(\Vkt x_i)=\Vkt x'_i$ for $i=1,2,3$ and $\beta(\Vkt x_j)=\Vkt x'_j$ for $j=i+3$. Moreover, the platform anchor points are not collinear.
     \item 
    There exists an indirect isometry $\alpha$ of $\RR^2$  and a direct isometry $\beta$ of $\RR^2$ with $\alpha(\Vkt x_i)=\Vkt x'_i$ for $i=1,2,3$ and $\beta(\Vkt x_j)=\Vkt x'_j$ for $j=i+3$. Moreover, the base anchor points are not collinear.
\end{enumerate}
 
As the sets C and D correspond to each other by swapping the roles of the platform and the base, we can restrict to the discussion of the sets A, B and C. In Section \ref{sec:parametrisation}, we parametrize these three sets of pairs of incongruent realisations, where the parameters also include the relative orientation of 
$\mathcal{G}(\Vkt X)$ and $\mathcal{G}(\Vkt X')$.

\section{Parametrizing the Sets of Incongruent 
Realisation-Pairs}\label{sec:parametrisation} 

\subsection{Set A}
Without loss of generality (w.l.o.g.) we can always assume that the two incogruent realisations $\mathcal{G}(\Vkt X)$ and $\mathcal{G}(\Vkt X')$  have the same base anchor points; i.e.\ $\Vkt x_i=\Vkt x'_i$ for $i=1,2,3$. 
Then the three points $\Vkt x_4,\Vkt x_5,\Vkt x_6$ and the three points $\Vkt x'_4,\Vkt x'_5,\Vkt x'_6$ are either related by (I) a rotation or (II) a translation. 

\subsubsection{(I) Rotation} 
As the centre $\Vkt u$ of rotation cannot be identical with all three points $\Vkt x_4,\Vkt x_5,\Vkt x_6$, we can assume w.l.o.g.\ that after a maybe necessary relabelling of the points $\Vkt u\neq \Vkt x_4$ holds true. 
Moreover, we can define the unit length as the distance between these two points. Now we can select our reference frame in a way that $\Vkt u$ is its origin and that $\Vkt x_4$ is located on the $y$-axis, which yields:
$\Vkt u=(0,0)^{\mathrm T}$ and $\Vkt x_4 = (0,1)^{\mathrm T}$.

Then the points $\Vkt x'_j$ for $j=4,5,6$ can be obtained from $\Vkt x_j$ by a rotation about the origin parametrized by the Blaschke-Gr\"unwald parameters $(e_0:e_1:0:0)\neq (1:0:0:0)$. Therefore we can set w.l.o.g.\ $e_1=1$.

 As we have now parametrized our two poses of the platform we remain with the parametrisation of the base points. This has to be done in a way that the leg lengths coincide; i.e. $r_i=r'_i$ for $i=1,2,3$. 

\begin{enumerate}[$i)$]
\item General case; $\Vkt x_{5}\neq \Vkt x'_{5}$ and $\Vkt x_{6}\neq  \Vkt x'_{6}$: 
As in this case  $\Vkt x_{i+3}\neq \Vkt x'_{i+3}$ holds true for  $i=1,2,3$, the base anchor point $\Vkt x_i=\Vkt x'_i$ has to be located on the bisecting line. Using the free parameter $l_i$, the base anchor point $\Vkt x_i$ can be written as
\begin{equation}
\Vkt x_i=\Vkt m_i+l_i\,\Vkt n_i,
\qquad i=1,2,3,
\label{eq:base_on_bisector}
\end{equation}
where $\Vkt m_i$ denotes the midpoint of the segment
$\Vkt x_{i+3}\Vkt x'_{i+3}$ and $\Vkt n_i$ is a direction vector
perpendicular to this segment with
$ \Vkt n_i=(b_{i+3}-b'_{i+3}, a'_{i+3}-a_{i+3})^{\mathrm T}$.

Finally, we apply a rotation about the origin to $\mathcal{G}(\Vkt X)$, parametrized by the Blaschke-Gr\"unwald parameters $(f_0 : f_1 : 0 : 0)$ using Eq.~(\ref{eq:bg_transform}).
Then the averaged configuration $\mathcal{G}(\bar{\Vkt X})$ can be constructed in dependence of the 10 parameters
$e_0, f_0, f_1, a_5, b_5, a_6, b_6, l_1, l_2, l_3$.

\item 
Special case; $\Vkt x_{5}\neq \Vkt x'_{5}$ but $\Vkt x_{6}=  \Vkt x'_{6}$: 
This case has to be discussed separately as for  $\Vkt x_{6}= \Vkt x'_{6}$ the bisecting construction is not defined. But $\Vkt x_{6}= \Vkt x'_{6}$ can only hold if  $\Vkt x_{6}$ is located in the centre of rotation; i.e.\  $\Vkt x_{6}=(0,0)^{\mathrm T}$. 
As the corresponding base anchor point $\Vkt x_3=\Vkt x'_3$ can be located anywhere, the averaged configuration
 $\mathcal{G}(\bar{\Vkt X})$ can then be constructed in dependence of the 9 parameters
$e_0, f_0, f_1, a_3, b_3, a_5, b_5, l_1, l_2$.

\item Very special case; $\Vkt x_{5}= \Vkt x'_{5}$ and $\Vkt x_{6}= \Vkt x'_{6}$:
Similar considerations as in the special case discussed above implies a construction of the averaged configuration
 $\mathcal{G}(\bar{\Vkt X})$  in dependence of the 8 parameters
$e_0, f_0, f_1, a_2, b_2, a_3, b_3, l_1$.

\end{enumerate}

\subsubsection{(II) Translation} 
In this case we can assume w.l.o.g.\ that the translation is done in $x$-direction and that it defines the unit length. Therefore the translation vector equals $(1,0)$  and we can select $\Vkt x_4$ as the origin of our reference frame; i.e.\ $\Vkt x_4 = (0,0)^{\mathrm T}$.

Then the points $\Vkt x'_j$ for $j=4,5,6$ can be obtained from $\Vkt x_j$ by a translation written as 
$(1:0:0:1/2)$ in terms of Blaschke-Gr\"unwald parameters {(\ref{eq:bg_transform}) }.  

As in this case  $\Vkt x_{i+3}\neq \Vkt x'_{i+3}$ holds true for  $i=1,2,3$, the base anchor point $\Vkt x_i=\Vkt x'_i$ has to be located on the bisecting line and we get again Eq.\ (\ref{eq:base_on_bisector}). 

Finally, we apply again a rotation about the origin to $\mathcal{G}(\Vkt X)$, parametrized by the Blaschke-Gr\"unwald parameters $(f_0 : f_1 : 0 : 0)$ using Eq.~(\ref{eq:bg_transform}).
Then the averaged configuration $\mathcal{G}(\bar{\Vkt X})$ can  be constructed in dependence of the 9 parameters
$f_0, f_1, a_5, b_5, a_6, b_6, l_1, l_2, l_3$.

\subsection{Set B} 

Here the parametrisation can be done in a similar manner with the sole difference that we apply to $\mathcal{G}(\Vkt X')$ a reflection on the $x$-axis; which corresponds to the mapping $(a'_k,b'_k)^{\mathrm T}\mapsto (a'_k,-b'_k)^{\mathrm T}$ for $k=1,\ldots, 6$. 

Note that we also have to distinguish the same cases as for the parametrisation of set A, with exception of Case ($I_{iii}$) as the platform points are assumed to be not collinear.

\subsection{Set C} 

Without loss of generality, we can always assume that the two incongruent realisations $\mathcal{G}(\Vkt X)$ and $\mathcal{G}(\Vkt X')$  have the same base anchor points; i.e., $\Vkt x_i=\Vkt x'_i$ for $i=1,2,3$. 
 
Then the three points $\Vkt x_4,\Vkt x_5,\Vkt x_6$ and the three points $\Vkt x'_4,\Vkt x'_5,\Vkt x'_6$ are either related by (I) a glide-reflection or (II) a pure reflection.

\subsubsection{(I) Glide-reflection}\label{sec:glide}

As not all three points of the triangle $\Vkt x_4,\Vkt x_5,\Vkt x_6$ can be located on the axis $\ell$ of the glide-reflection, we can assume w.l.o.g.\ that after a maybe necessary relabelling of the points $\Vkt x_4\notin\ell$ holds true. 
Moreover, we can define the unit length as the distance between $\Vkt x_4$ and its pedal point $\Vkt u$ on $\ell$. Now we can select our reference frame in a way that $\Vkt u$ is its origin, $\ell$ equals the $x$-axis and that $\Vkt x_4$ is located on the positive $y$-axis which yields:
$\Vkt u=(0,0)^{\mathrm T}$ and $\Vkt x_4 = (0,1)^{\mathrm T}$.

Then the points $\Vkt x'_j$ for $j=4,5,6$ can be obtained from $\Vkt x_j$ by the glide-reflection 
$(a'_j,b'_j)^{\mathrm T}= (a_j+d,-b_j)^{\mathrm T}$ where $d\neq 0$ is the translation distance along $\ell$.
As in this case  $\Vkt x_{i+3}\neq \Vkt x'_{i+3}$ holds true for  $i=1,2,3$, the base anchor point $\Vkt x_i=\Vkt x'_i$ has to be located on the bisecting line and we get again Eq.\ (\ref{eq:base_on_bisector}). 

Finally, we apply again a rotation about the origin to $\mathcal{G}(\Vkt X)$, parametrized by the Blaschke-Gr\"unwald parameters $(f_0 : f_1 : 0 : 0)$ using Eq.~(\ref{eq:bg_transform}).
Then the averaged configuration $\mathcal{G}(\bar{\Vkt X})$ can  be constructed in dependence of the 10 parameters
$d, f_0, f_1, a_5, b_5, a_6, b_6, l_1, l_2, l_3$.

\subsubsection{(II) Reflection}
We have to distinguish the following three subcases: 
\begin{enumerate}[$i)$]
\item General case; $\Vkt x_{5}\neq \Vkt x'_{5}$ and $\Vkt x_{6}\neq  \Vkt x'_{6}$:  
The parametrisation of this case can be obtained from the one of the glide-reflection by setting $d=0$. Therefore the averaged configuration $\mathcal{G}(\bar{\Vkt X})$ can  be constructed in dependence of the 9 parameters
$f_0, f_1, a_5, b_5, a_6, b_6, l_1, l_2, l_3$.

\item Special case; $\Vkt x_{5}\neq  \Vkt x'_{5}$ but $\Vkt x_{6}=  \Vkt x'_{6}$: 
This case has to be discussed separately as for  $\Vkt x_{6}= \Vkt x'_{6}$ the bisecting construction is not defined. But $\Vkt x_{6}= \Vkt x'_{6}$ can only hold if  $\Vkt x_{6}$ is located on the axis $\ell$ of the reflection; i.e.\ $\Vkt x_{6}=(a_6,0)^{\mathrm T}$. 
As the corresponding base anchor point $\Vkt x_3=\Vkt x'_3$ can be located anywhere, the averaged configuration
 $\mathcal{G}(\bar{\Vkt X})$ can then be constructed in dependence of the 9 parameters
$f_0, f_1, a_3, b_3, a_5, b_5, a_6, l_1, l_2$. 
\item 
Very special case; $\Vkt x_{5}= \Vkt x'_{5}$ and $\Vkt x_{6}= \Vkt x'_{6}$: 
Similar considerations as in the special case discussed above implies a construction of the averaged configuration
 $\mathcal{G}(\bar{\Vkt X})$  in dependence of the 9 parameters
$f_0, f_1, a_2, b_2, a_3, b_3, a_5, a_6, l_1$.
\end{enumerate}

\section{Results}\label{sec:results}

Using the parametrisation of the sets A, B and C of pairs of incongruent realisations given in Section \ref{sec:parametrisation}, we can formulate the following theorems for an averaged configuration $\mathcal{G}(\bVkt X)$ (cf.\ Definition \ref{def:averconf}) 
of two incongruent realisations 
$\mathcal{G}(\Vkt X)$ and $\mathcal{G}(\Vkt X')$ of a pin-jointed bar-plate framework, which represents a non-degenerated planar $3$-RPR  parallel manipulator with locked P joints.

\begin{thm}\label{thm:setB}
If the pair $\mathcal{G}(\Vkt X)$ and $\mathcal{G}(\Vkt X')$ belongs to the set B, then
    $\mathcal{G}(\bVkt X)$ is in general first order rigid (flexion order 0).
    $\mathcal{G}(\bVkt X)$ has a flexion order of 1 in case ($\Gamma$) iff the polynomial $f_0A_{\Gamma} + f_1B_{\Gamma}$ vanishes for $\Gamma=I_i,I_{ii},II$ with:
\begin{align*}
    A_{I_i}= \,\,&
(e_0-2e_1l_2 )(a_5e_0 - b_5e_1) ( l_1 - l_3 )
-
(e_0-2e_1l_3)(a_6e_0 - b_6e_1) ( l_1 - l_2 )
+\\
&(e_0- 2 e_1 l_1)e_1 ( l_2 - l_3 ) \\
B_{I_i}= \,\,&
(e_0-2e_1l_3)(a_6e_1 + b_6e_0) ( l_1 - l_2 ) 
-(e_0-2e_1l_2)(a_5e_1 + b_5e_0) ( l_1 - l_3 )
+\\
&(e_0- 2 e_1 l_1)e_0 ( l_2 - l_3 )\\
A_{I_{ii}} = \,\,&(e_0 - 2 e_1 l_2)(a_5 e_0 - b_5 e_1) + (e_0 - 2 e_1 l_1)e_1 \\
B_{I_{ii}} = \,\,&(e_0 - 2 e_1 l_1)e_0-(e_0-2 e_1 l_2)(a_5 e_1 + b_5 e_0)
       \\
A_{II} = \,\,&a_5(l_1 - l_3) - a_6(l_1 - l_2) \\
B_{II} = \,\,&(l_2 - l_3)-b_5(l_1 - l_3) + b_6(l_1 - l_2) 
\end{align*}

\end{thm}
Therefore in all cases there exist in general a special relative orientation $f_0:f_1$ illustrated in Fig.\ \ref{fig:setB}, which implies that $\mathcal{G}(\bVkt X)$ has flexion order 1. 
Note that Theorem \ref{thm:setB} does not conflict with the averaging theorem (cf.\ Theorem \ref{thm:aver}), as we are considering a pin-jointed bar-plate framework.
\begin{thm}\label{thm:setC}
If the pair $\mathcal{G}(\Vkt X)$ and $\mathcal{G}(\Vkt X')$ belongs to the set C, then
\begin{enumerate}
    \item 
    $\mathcal{G}(\bVkt X)$ is in general first order rigid (flexion order 0) for the cases ($I$) and ($II_{ii}i$). $\mathcal{G}(\bVkt X)$ has a flexion order of 1 iff $f_0=0$ holds for Case (I) and $f_0b_2b_3=0$ for  Case $(II_{iii})$.
    \item
    $\mathcal{G}(\bVkt X)$ is at least of flexion order 1 for the Cases ($II_i$) and ($II_ii$). For Case ($II_i$) all six points are collinear which corresponds to a singular point of $V_1$. Assume 
    $\mathcal{G}(\bVkt X)$ does not correspond to a singular point of $V_1$, then $\mathcal{G}(\bVkt X)$ has a flexion order of 2 in Case ($II_{ii}$) iff the following polynomial $P_{II_{ii}}$ vanishes:
     \begin{align*}
    (f_1+2f_0l_2)(2f_1l_1 - f_0a_6^2)b_5  
    + 4b_5f_0f_1(l_1 - l_2)a_6+ \\
    (f_1+2f_0l_1)[f_0(a_5-a_6)^2
    - 2f_1l_2b_5^2 
    - a_5b_5(f_1-2f_0l_2)]
    \end{align*}
\end{enumerate}
\end{thm}
 Note that $(f_0:f_1)=(0:1)$ implies that the base of $\mathcal{G}(\bVkt X)$ degenerate\footnote{We have not excluded the degenerated cases for  $\mathcal{G}(\bVkt X)$, as this degeneration does not violate the combinatorial structure of a pin-jointed bar-plate framework.} to a point. For $b_k=0$ in ($II_{iii}$) the first leg and the $k$th leg are collinear for $k=2,3$. 
 The vanishing of the remaining factor of Case ($II_{ii}$) is illustrated in Fig.\ \ref{fig:setCA}a.

\begin{thm}\label{thm:setA}
If the pair $\mathcal{G}(\Vkt X)$ and $\mathcal{G}(\Vkt X')$ belongs to the set A, then
    $\mathcal{G}(\bVkt X)$ is at least of flexion order 1. 
    Assume    $\mathcal{G}(\bVkt X)$ does not correspond to a singular point of $V_1$, then $\mathcal{G}(\bVkt X)$
    has a flexion order of 2 
    iff the following polynomial vanishes for:
     \begin{align}
     &\text{Case ($I_i$)} &\quad &f_0(f_0 e_0 + f_1e_1 )P_{I_i} \label{eq:1}\\
      &\text{Case ($I_{ii}$)} &\quad &  f_0(f_0 e_0 + f_1e_1 ) P_{I_{ii}} \label{eq:2} \\
       &\text{Case ($I_{iii}$)} &\quad & f_0(e_0 f_0 + e_1 f_1) (a_2 b_3 - a_3 b_2)(e_0 - 2 e_1 l_1) \label{eq:3}\\
       &\text{Case (II)} &\quad &f_0 P_{II} \label{eq:4}
     \end{align}
with 
\[
\begin{aligned}
P_{I_i}
=\,\, &(e_0 - 2e_1 l_2)\, a_6\, (a_5^2 + b_5^2)\, (f_1 + 2l_1 f_0)(f_1 + 2l_3 f_0)- \\
&(e_0 - 2e_1 l_3)\, a_5\, (a_6^2 + b_6^2)\, (f_1 + 2l_1 f_0)(f_1 + 2l_2 f_0)+ \\
&(e_0 - 2e_1 l_1)\, (a_5 b_6 - a_6 b_5)\, (f_1 + 2l_2 f_0)(f_1 + 2l_3 f_0)
\\[1em]
P_{I_{ii}} =\,\, &(e_0 - 2 e_1 l_2)(a_5^2 + b_5^2)(a_3 e_0 + b_3 e_1)(f_1 + 2 f_0 l_1)- \\
& (e_0 - 2 e_1 l_1)\big[ (a_3 b_5 - a_5 b_3)e_0
    + (a_3 a_5 + b_3 b_5)e_1 \big](f_1 + 2 f_0 l_2)
\\[1em]
P_{II} =\,\,
&a_5(l_1 - l_3)(f_1+2f_0l_2)
-a_6(l_1 - l_2)(f_1+2f_0l_3)
\end{aligned}
\]
\end{thm}
 Note that $(f_0:f_1)=(0:1)$ implies again the degeneration of the base to a point.  For $(f_0:f_1)=(-e_1:e_0)$ the platform degenerates into a point. If in Case ($I_{iii}$) $e_0 = 2 e_1 l_1$ holds true, then the platform can rotate about the point $\Vkt x_1=\Vkt x_5=\Vkt x_6$, which also holds true for the averaged framework; i.e.\ $\bVkt x_1=\bVkt x_5=\bVkt x_6$. 
 If in Case ($I_{iii}$) $a_2 b_3 = a_3 b_2$ holds true ($\Rightarrow$ $\Vkt x_5=\Vkt x_6, \Vkt x_2, \Vkt x_3$ are collinear) then the averaged configuration has flexion order 2 independently of the chosen orientation $(f_0:f_1)$. This case is illustrated in Fig.\ \ref{fig:setCA}b and the vanishing of the polynomials $P_{I_i}$, $P_{I_{ii}}$ and $P_{II}$ is visualized in \medskip Fig.\ \ref{fig:setA}.

The detailed proofs of the Theorems \ref{thm:setB}--\ref{thm:setA} are given in the corresponding Appendices \ref{sec:proofA}--\ref{sec:proofC}. In the following we only outline the computational pipeline of these proofs: 

Using the parametrisation of the sets A, B and C of pairs of incongruent realisations given in Section \ref{sec:parametrisation}, we compute the average configuration $\mathcal{G}(\bVkt X)$, which is used for entering the approach described in the Subsections \ref{sec:direct} and \ref{sec:order12}. By setting the leg lengths $r_i=\|\bVkt x_i-\bVkt x_{i+3}\|$ for $i=1,2,3$ the considered configuration  $\mathcal{G}(\bVkt X)$ corresponds to the solution of the direct kinematic problem given by 
 $(q_0:\ldots:q_3)=(1:0:0:0)$. 
Therefore we have to evaluate the polynomials $s,s_0,\ldots,s_3$ at this pose. 
From the resulting expressions we cancel out all factors implying a leg of length zero, as this would violate the combinatorial structure of $\mathcal{G}(\bVkt X)$. 

The polynomials given in the Theorems \ref{thm:setC} and \ref{thm:setA} which imply a configuration of flexion order 2 belong to the greatest common divisor (gcd) of the generators $s,s_0,\ldots, s_3$ of $\mathcal{I}_2$ given in Eq.\ (\ref{eq:ideal}). By deleting all factors of $s$ and $s_i$ which also show up in this gcd, we obtain the reduced polynomials $s^*$ and $s_i^*$. It can be checked (via Maple) that the variety of $\mathcal{I}_2^*=\langle s^*, s_0^*, s_1^*, s_2^*, s_3^* \rangle$ is a subset of the variety of singular points of $V_1$ determined by the five equations $s=\nabla s=0$. 
\begin{remark}
    Flexion order 2 was geometrically characterized by Stachel \cite{stachel:planar} for the bar-joint framework interpretation of planar parallel manipulators. We verified our results obtained for $\mathcal{G}(\bVkt X)$ with flexion order 2 illustrated in Figs.\ \ref{fig:setA} and \ref{fig:setCA} by checking for this geometric property (for details see Figs.\ \ref{fig:stachel}, \ref{fig:stacheltest1}, \ref{fig:setAstachelcheck2}, \ref{fig:trancheck}). \hfill $\diamond$
\end{remark}

\begin{figure}[t]
\begin{center} 
  \begin{overpic}
   [width=75mm]{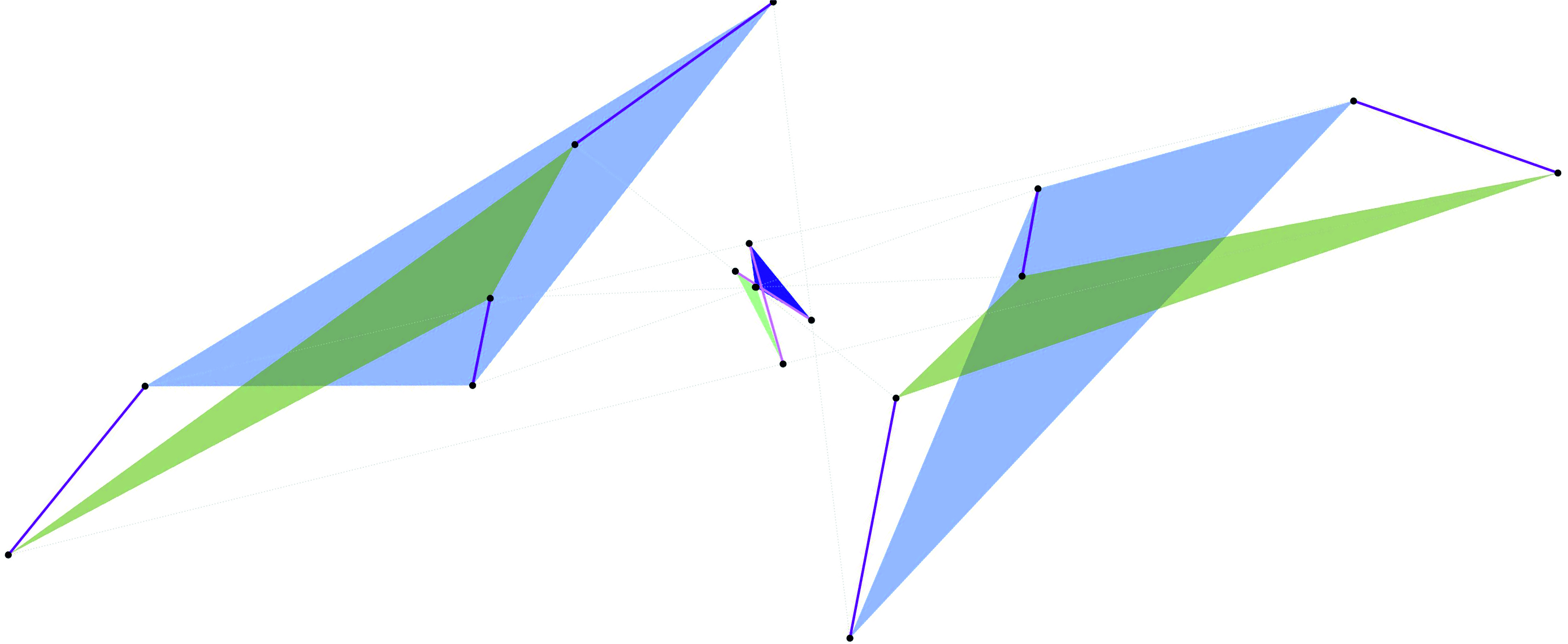}
   \begin{scriptsize}
   \put(-4,0){a)} 
\put(52,18){$\bVkt x_2$}   
\put(47,27){$\bVkt x_3$}
\put(50,23){$\bVkt x_1$}
\put(42,24){$\bVkt x_5$}
\put(47,15){$\bVkt x_6$}
\put(43,21){$\bVkt x_4$}
\end{scriptsize}
\end{overpic} 
\begin{overpic}
   [width=75mm]{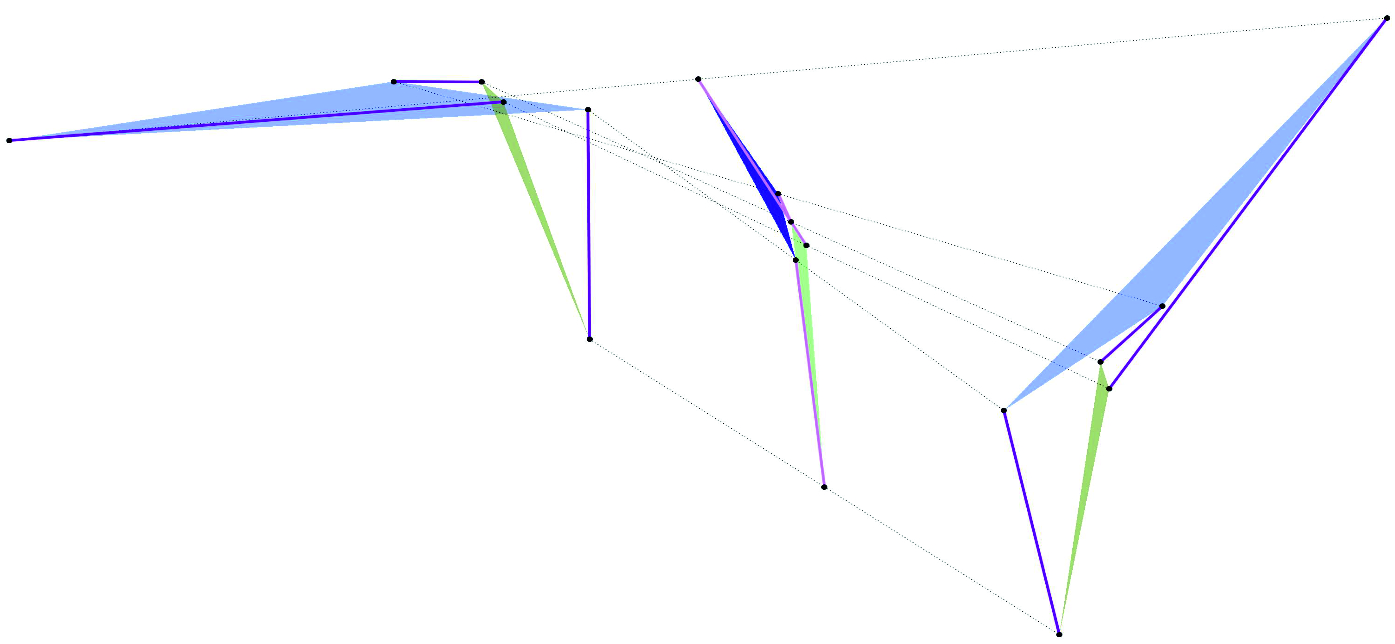}
   \begin{scriptsize}
\put(-4,0){b)}
\put(53,26){$\bVkt x_2$}   
\put(48,42){$\bVkt x_1$}
\put(56,33){$\bVkt x_3$}
\put(58,8){$\bVkt x_5$}
\put(59,27){$\bVkt x_4$}
\put(58,30){$\bVkt x_6$}
\end{scriptsize}
\end{overpic} 
\begin{overpic}
   [width=75mm]{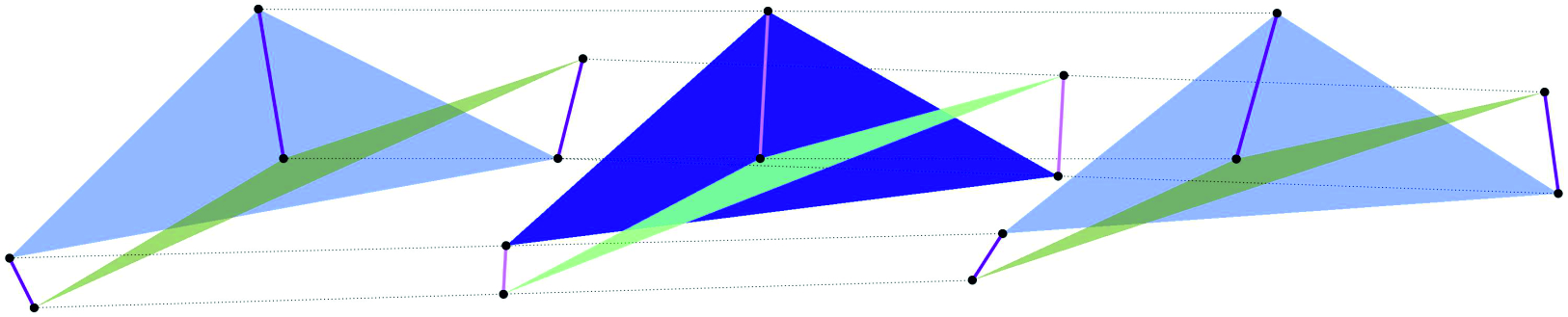}
   \begin{scriptsize}
\put(-4,0){c)}
\put(28,5){$\bVkt x_1$}   
\put(64,11){$\bVkt x_2$}
\put(45,20){$\bVkt x_3$}
\put(28,1){$\bVkt x_4$}
\put(64,16){$\bVkt x_5$}
 \put(44,11){\contour{white}{$\bVkt x_6$}}
\end{scriptsize}
\end{overpic} 
\end{center}	
\caption{Illustration of Theorem \ref{thm:setA}: 
    (a) Solution of Case ($I_i$); for the other one see Fig.\ \ref {fig:avg3rpr}a. (b) Case ($I_{ii}$), (c) In Case ($II$) the three legs of $\mathcal{G}(\bVkt X)$ are always parallel.}
\label{fig:setA}
\end{figure} 

\begin{figure}[t]
\begin{center}
\begin{overpic}
   [width=70mm]{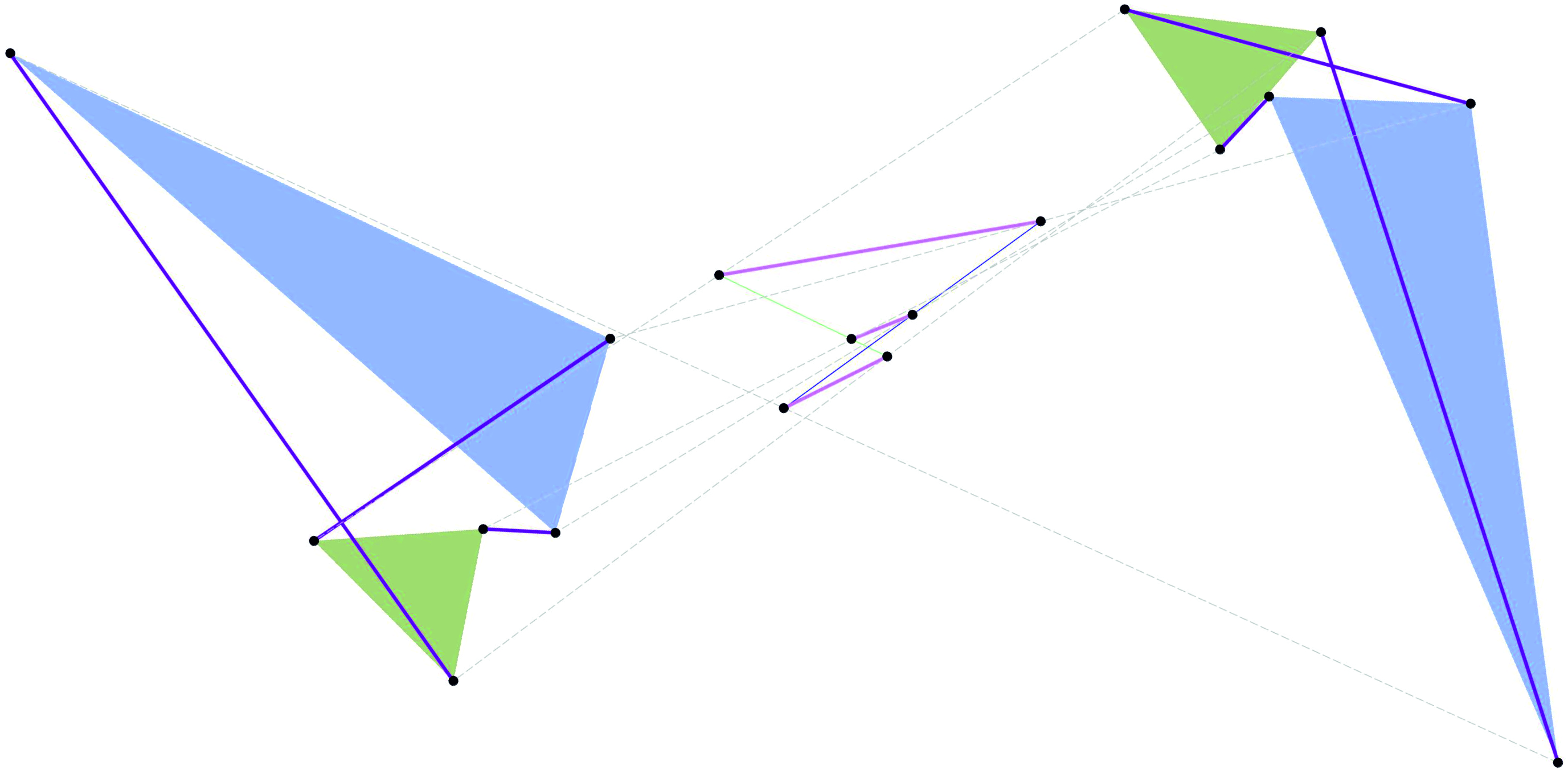}
   \begin{scriptsize}
\put(-4,0){a)}
\put(49,26){$\bVkt x_5$}   
\put(48,20){$\bVkt x_1$}
\put(65,36){$\bVkt x_3$}
\put(56,23){$\bVkt x_4$}
\put(59,28){$\bVkt x_2$}
\put(41,31){$\bVkt x_6$}
\end{scriptsize}
\end{overpic} 
  \begin{overpic}
   [width=70mm]{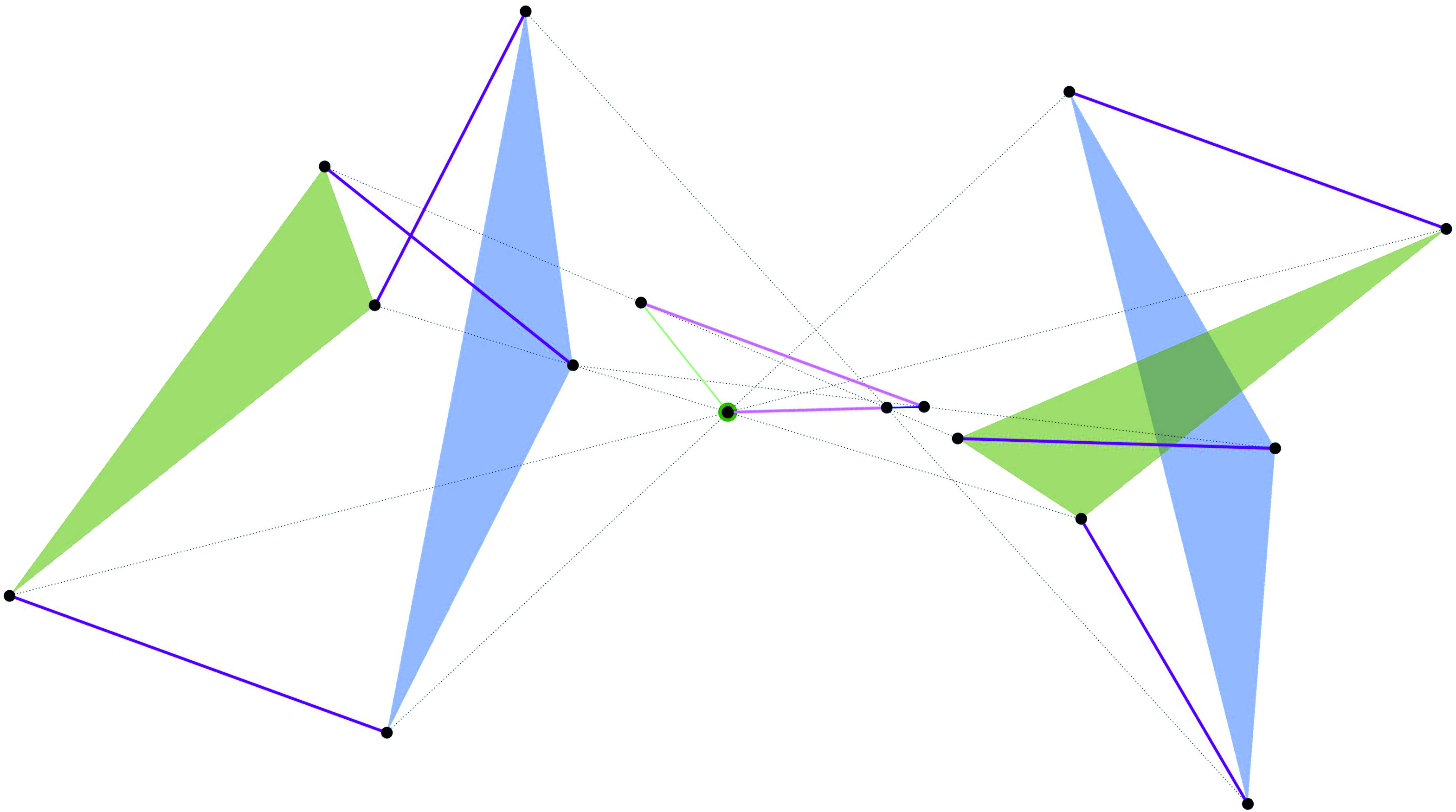}
   \begin{scriptsize}
\put(-4,0){b)}
\put(58,29){$\bVkt x_2$}   
\put(63,29){$\bVkt x_1$}
\put(37,24){$\bVkt x_3=\bVkt x_5=\bVkt x_6$}
\put(42,37){$\bVkt x_4$}
\end{scriptsize}
\end{overpic} 
  \begin{overpic}
    [width=70mm]{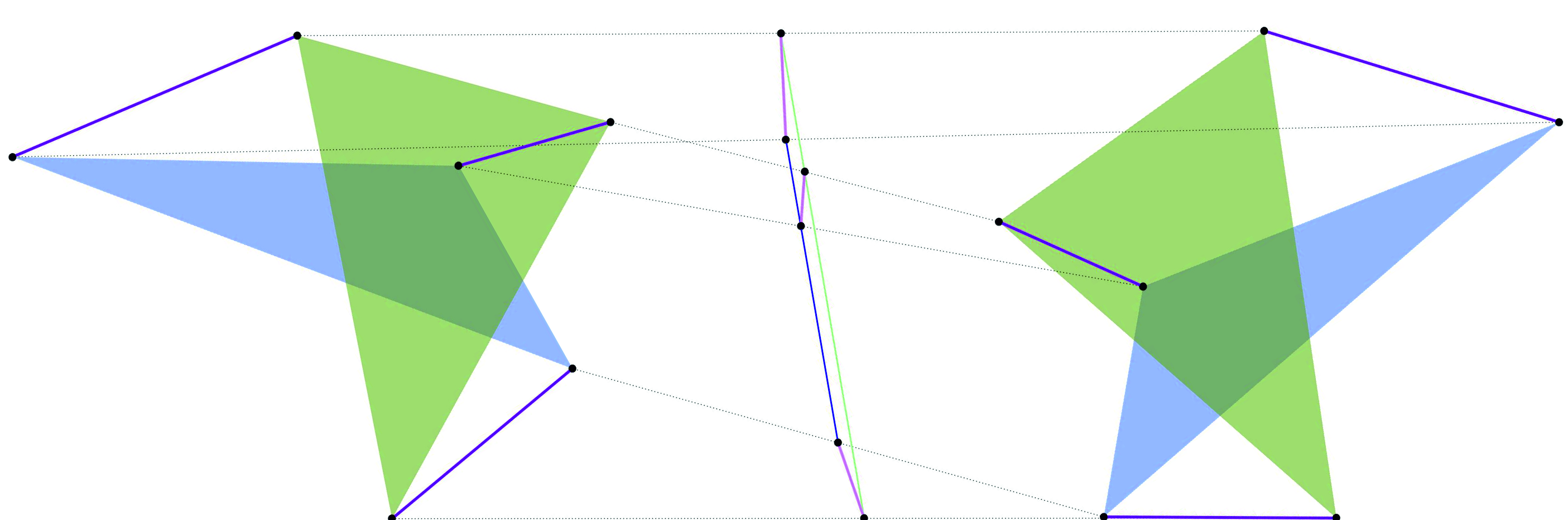}
   \begin{scriptsize}
\put(-4,0){c)}
\put(49,5){$\bVkt x_1$}   
\put(46,18){$\bVkt x_2$}
\put(45,24){$\bVkt x_3$}
\put(50,1){$\bVkt x_4$}
\put(52,22){$\bVkt x_5$}
\put(45,31){$\bVkt x_6$}
\end{scriptsize}
\end{overpic} 
  \end{center}	
\caption{Illustration of Theorem \ref{thm:setB}: (a) Case ($I_i$), (b) Case ($I_{ii}$), (c) Case ($II$). Note that if the three platform/base points of $\mathcal{G}(\Vkt X)$ and $\mathcal{G}(\Vkt X')$ are related by an indirect isometry, then the averaged three platform/base points of $\mathcal{G}(\bVkt X)$ are collinear (cf.\ Lemma \ref{lem:coll}).}
  \label{fig:setB}
\end{figure}    

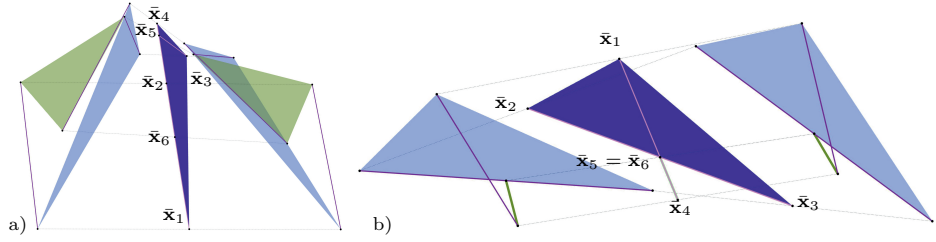
\begin{figure}[t]
\begin{center}
\hfill
\begin{overpic}[height=30mm]{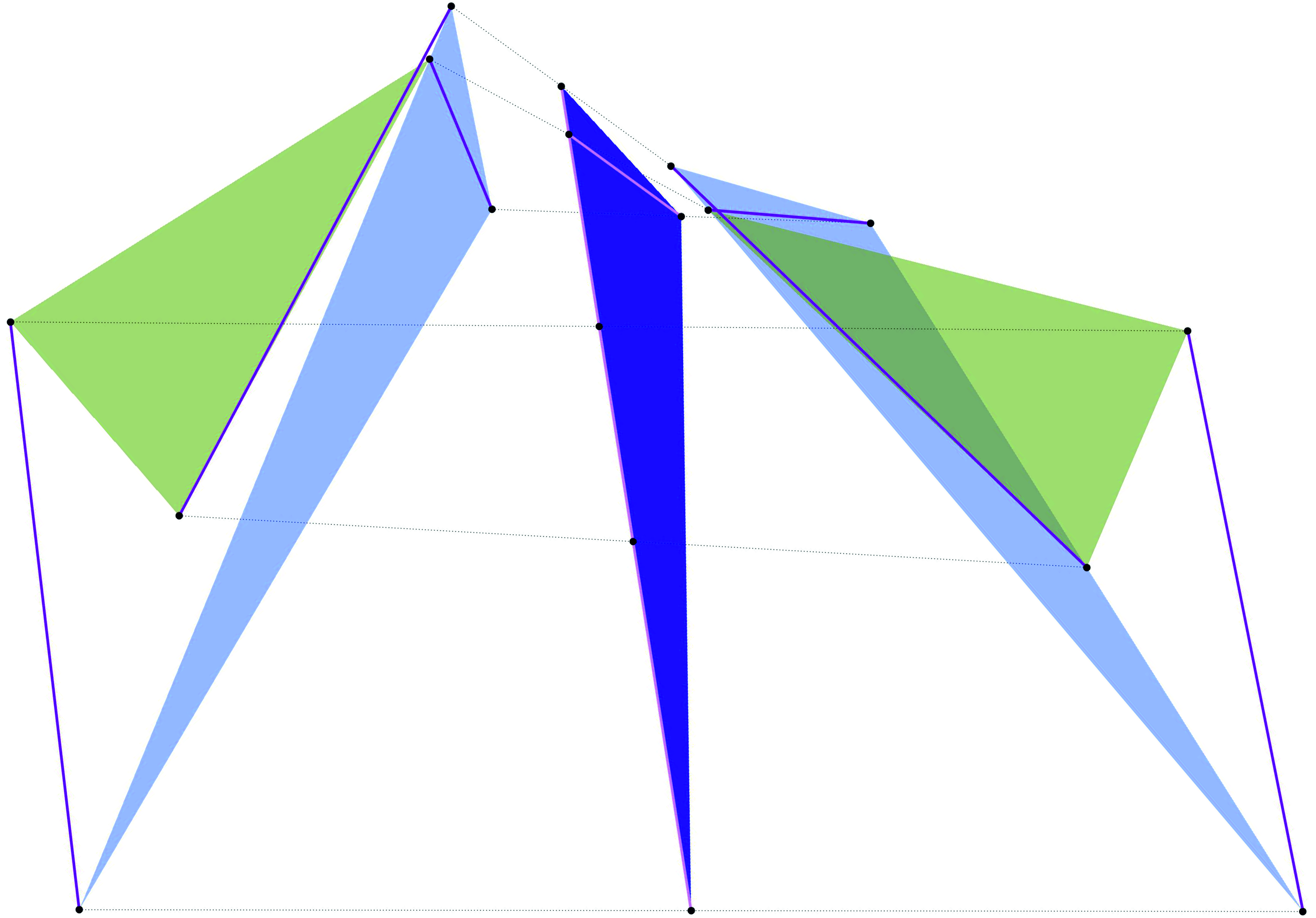}
  \begin{scriptsize}
    \put(-3,0){a)}
    \put(44.5,3){$\bVkt x_1$}
    \put(53,45){$\bVkt x_3$}
    \put(38,44){$\bVkt x_2$}
    \put(40,65){$\bVkt x_4$}
    \put(36,60){$\bVkt x_5$}
    \put(40,27){$\bVkt x_6$}
  \end{scriptsize}
\end{overpic}
\hfill
\begin{overpic}[height=28mm]{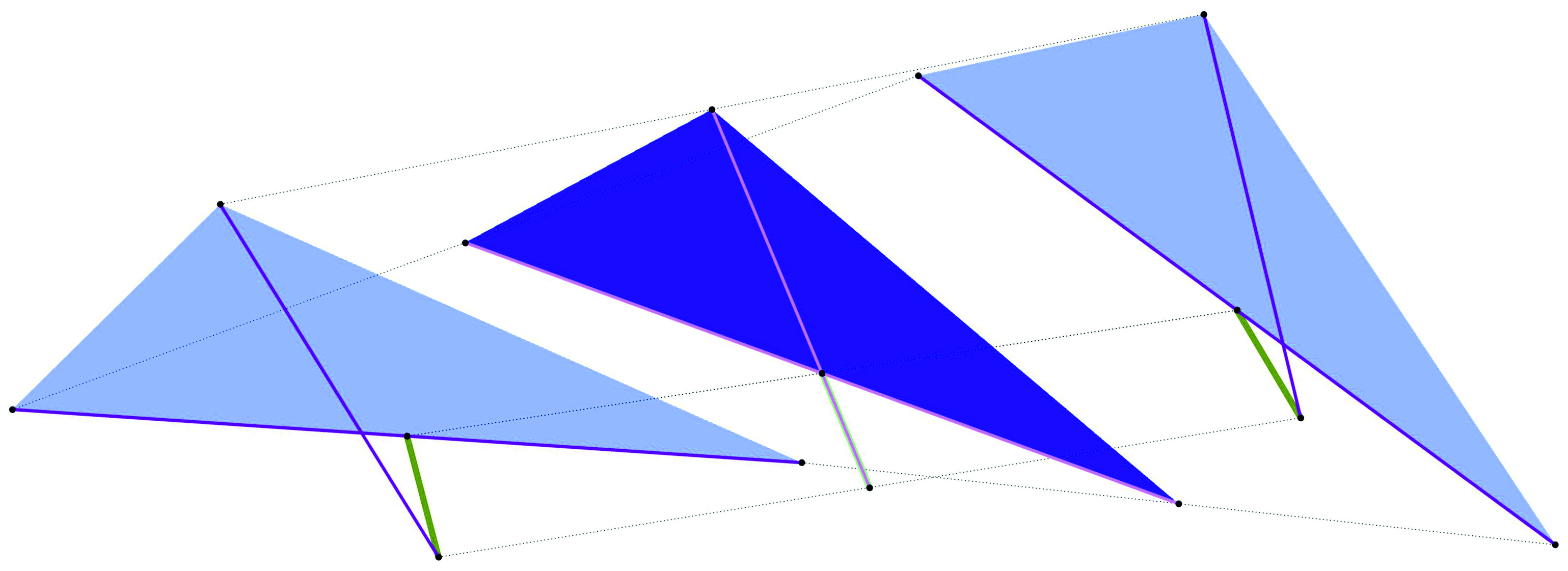}
 \begin{scriptsize}
\put(3,0){b)}
    \put(42,32){$\bVkt x_1$}
    \put(24,21){$\bVkt x_2$}
    \put(76,4){$\bVkt x_3$}
    \put(54,3){$\bVkt x_4$}
    \put(38,11){$\bVkt x_5=\bVkt x_6$}
    \end{scriptsize}
\end{overpic}
\end{center}
\caption{(a) Illustration of Case ($II_{ii}$) of Theorem \ref{thm:setC}: Five points are always collinear in the averaged configuration $\mathcal{G}(\bVkt X)$,
(b) Illustration of Case ($I_{iii}$) of Theorem \ref{thm:setA}: In Case ($I_{iii}$), the second and third legs lie along the same line, while the first leg is aligned with the platform.}
\label{fig:setCA}
\end{figure}

\section{Conclusion and Future Research}\label{sec:conc}

In Theorems \ref{thm:setB}--\ref{thm:setA} we determined the relative orientation between two incongruent realisations $\mathcal{G}(\Vkt X)$ and $\mathcal{G}(\Vkt X')$ of a pin-jointed bar-plate framework, which represents a non-degenerated  planar $3$-RPR parallel manipulator with locked P joints, in a way that the flexion order of the averaged $\mathcal{G}(\bVkt X)$ configuration increases. 
For some special cases, we were able to identify a simple geometric interpretation of these special relative orientations, but for the general cases, such characterisations remain open.

The averaging technique also works for spherical frameworks \cite{izmestiev:2009} and therefore the presented approach can also be 
adapted for spherical $3$-RPR parallel mechanisms. In this case it should be possible to construct averaged configurations with a flexion order of at least 4 by a suitable choice of relative orientations. This study is dedicated to future research. 

Theoretically the presented methodology can also be applied to spatial analogues of planar $3$-RPR parallel manipulators. However, also for these so-called Stewart-Gough platforms only a  flexion order of at least 4 for the averaged configurations can in general be achieved by a suitable choice of relative orientations.

\paragraph{{\bf Acknowledgments}}
This research was funded in whole or in part by the Austrian Science Fund (FWF) [grant DOI 10.55776/PAT1144724]. For open access purposes, the author has applied a CC BY public copyright license to any author accepted manuscript version arising from this submission.

\newpage

\appendix

\section{Proof of the results for Set A (Theorem \ref{thm:setA})}\label{sec:proofA} 

\subsection{Rotation}
To extract the second higher-order flexibility component, we compute the
greatest common divisor
\[
P = \gcd(s,s_1,s_2,s_3,s_4).
\]

\subsubsection{General Case:} $\Vkt x_{5}\neq \Vkt x'_{5}$ and $\Vkt x_{6}\neq \Vkt x'_{6}$ \newline
\noindent
In this case, the factorisation of $P$ yields three homogenous factors in $f_0,f_1$ given in Eq.\ (\ref{eq:1}), where two are linear and one is quadratic. Hence, up to four solutions over $\CC$ may occur, corresponding to the relative orientations of $\mathcal{G}(\Vkt X)$ and $\mathcal{G}(\Vkt X')$ for which $\mathcal{G}(\bVkt X)$ has flexion order of at least $2$.

In the following we describe the corresponding geometric configurations arising when each factor vanishes:
\begin{enumerate}
\item[(a)]$f_0 e_0 + f_1 e_1 = 0$: 
If this factor vanishes, then the platform of $\mathcal{G}(\bVkt X)$ degenerates to a single point.

\item[(b)]$
f_0 = 0 $: 
If this factor vanishes, then the base of $\mathcal{G}(\bVkt X)$ degenerates to a single point.

\item[(c)]$P_{I_i} = 0$ and $f_0(f_0 e_0 + f_1 e_1)\neq 0$:
After rearranging $P_{I_i}$ we can rewrite it as:
\begin{align}\label{eq:struc}
\begin{aligned}
P_{I_i}
=\,\, &T_1 (f_1 + 2l_1 f_0)(f_1 + 2l_3 f_0)- \\
  &T_2 (f_1 + 2l_1 f_0)(f_1 + 2l_2 f_0)+ \\
  &T_3 (f_1 + 2l_2 f_0)(f_1 + 2l_3 f_0)
\end{aligned}
\quad \Bigg| \quad
\begin{aligned}
T_1 &= (e_0 - 2e_1 l_2)\, a_6\, (a_5^2 + b_5^2) \\
T_2 &= (e_0-2e_1 l_3)\, a_5\, (a_6^2 + b_6^2) \\
T_3 &= (e_0 - 2e_1 l_1)\, (a_5 b_6 - a_6 b_5)
\end{aligned}
\end{align}
In the generic case where none of the terms 
\[
T_j (f_1 + 2l_k f_0)(f_1 + 2l_j f_0) 
\quad \text{for }\quad k,j \in \{1,2,3\},
\]
vanishes identically, the polynomial \( P_{I_i} \) is quadratic in \( (f_0, f_1) \).
Consequently, there exist parameter choices for which the equation \( P_{I_i} = 0 \)
admits two distinct real solutions, as illustrated in the following example.
\end{enumerate}

\begin{figure}[t]
\begin{center}
\begin{overpic}[height=40mm]{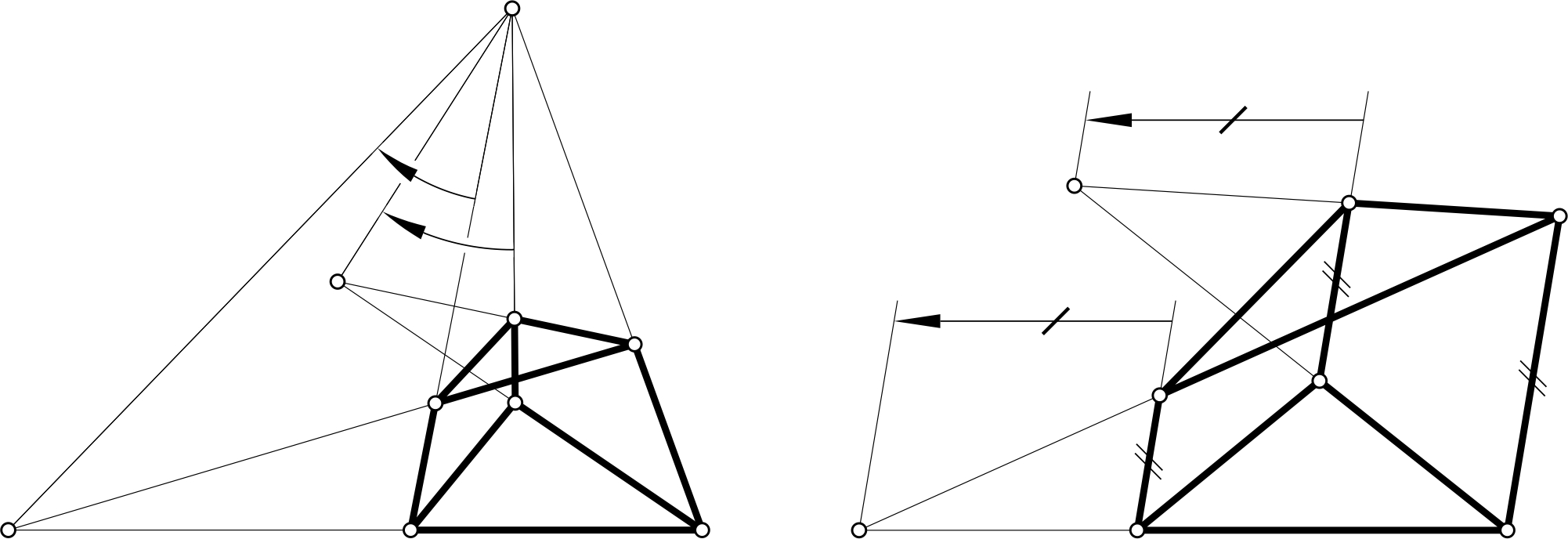}
   \begin{scriptsize}
   \put(-3,0){a)} 
   \put(34,33){$L$}
   \put(0,-3){$Q_{25}$}
   \put(19,14){$Q_{36}$}
   
    \put(44,-3){$\bVkt x_1$}
    \put(25,-3){$\bVkt x_2$}
    \put(32,5){$\bVkt x_3$}
    \put(42,12){$\bVkt x_4$}
    \put(33,15){$\bVkt x_6$}
    \put(24,9){$\bVkt x_5$}
    \put(25,21){$\alpha$}
    \put(28,17){$\beta$}
   \put(51,0){b)}
    \put(54,-3){$Q_{25}$}
   \put(61,22){$Q_{36}$}
   \put(72,-3){$\bVkt x_2$}
   \put(70,11){$\bVkt x_5$}
   \put(95,-3){$\bVkt x_1$}
   \put(97,22){$\bVkt x_4$}
   \put(83,6){$\bVkt x_3$}
   \put(83,23){$\bVkt x_6$}
\end{scriptsize}
\end{overpic}
\end{center}	
\caption{
Stachel's geometric characterization of flexion order 2 for the bar-joint framework interpretation of planar parallel manipulators with copunctal legs (a) and parallel legs (b). These are the original figures of \cite{stachel:planar} by courtesy of Hellmuth Stachel. \newline
(ad a) In the case where the three lines $[\bVkt{x}_1,\bVkt{x}_4]$, $[\bVkt{x}_2,\bVkt{x}_5]$ and $[\bVkt{x}_3,\bVkt{x}_6]$
 have the point $L$ in common the second order flexion condition is equivalent to
the equality of the two orientated angles $\alpha$ and $\beta$:
The angle $\alpha$ is determined by $[\bVkt{x}_2,\bVkt{x}_5]$ and $[L,Q_{25}]$ where $Q_{25}$ is the intersection point of  
$[\bVkt{x}_1,\bVkt{x}_2]$ and $[\bVkt{x}_4,\bVkt{x}_5]$. The angle $\beta$ is enclosed by $[\bVkt{x}_3,\bVkt{x}_6]$ and $[L,Q_{36}]$ where $Q_{36}$ is the intersection point of 
$[\bVkt{x}_1,\bVkt{x}_3]$ and $[\bVkt{x}_4,\bVkt{x}_6]$. \newline
(ad b) In the case where the three legs are parallel, the second order flexion condition is equivalent to the equality of oriented distances from $Q_{25}$ to the line $[\bVkt{x}_2,\bVkt{x}_5]$ and from  $Q_{36}$ to the line $[\bVkt{x}_3,\bVkt{x}_6]$.
}
  \label{fig:stachel}
\end{figure}

\begin{example}\label{ex:1}
   The present example corresponds to the configurations depicted in Fig.~\ref{fig:avg3rpr}(a) and Fig.~\ref{fig:setA}(a).

 We choose
$$
\Vkt x_5=(a_5,b_5)=(2,5), \qquad
\Vkt x_6=(a_6,b_6)=(-12,-6).
$$
For the initial relative orientation of
$\mathcal{G}(\Vkt X)$ and $\mathcal{G}(\Vkt X')$,
we take the Blaschke--Gr\"unwald parameters
$$
(e_0:e_1)=\left(\tfrac{24}{25}:\tfrac{7}{25}\right).
$$
Moreover, we choose the parameters describing the base anchor points on the
perpendicular bisectors as
$$
l_1=4, \qquad l_2=-2, \qquad l_3=\tfrac{7}{13}.
$$

Substituting these values into the polynomial $P_{\mathrm{I_i}}$, we obtain a
quadratic homogeneous equation in the orientation parameters $(f_0:f_1)$.
Fixing $f_0=1$, this reduces to a quadratic equation in $f_1$, whose two
real solutions are
$$
f_1^{(1)}
=
\tfrac{1}{139385930}
\sqrt{
9963395831860025
-209078895\,\sqrt{1900978050015889}
},
$$
and
$$
f_1^{(2)}
=
\tfrac{1}{139385930}
\sqrt{
9963395831860025
+209078895\,\sqrt{1900978050015889}
}.
$$
For both solutions in this example, the corresponding averaged configuration
$\mathcal{G}(\bVkt X)$ satisfies Stachel's geometric criterion (cf.\ Fig.\ \ref{fig:stachel}) and therefore
possesses flexion of order $2$. This is illustrated in Fig.~\ref{fig:stacheltest1}. 
\end{example} 

\begin{figure}[t]
\begin{center}
\quad
\begin{overpic}[width=0.6\linewidth]{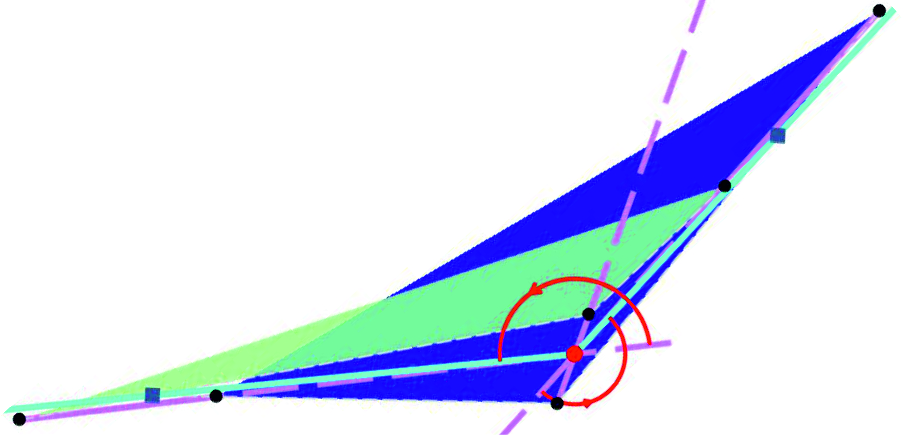}
  \begin{scriptsize}
   \put(-5,0){a)} 
   \put(64,5){$L$}
    \put(0,5){$\bVkt x_6$}
   \put(14,0){$Q_{36}$}
   \put(87,29){$Q_{25}$}
    \put(59,0){$\bVkt x_1$}
    \put(98,43){$\bVkt x_2$}
    \put(22,0){$\bVkt x_3$}
    \put(59,12){\contour{white}{$\bVkt x_4$}}
    
    \put(81,23){$\bVkt x_5$}
    \put(67,2){$\overline{\alpha}$}
    \put(57,18){\contour{white}{$\overline{\beta}$}}
   
\end{scriptsize}
\end{overpic}
\quad
\begin{overpic}[width=0.25\linewidth]{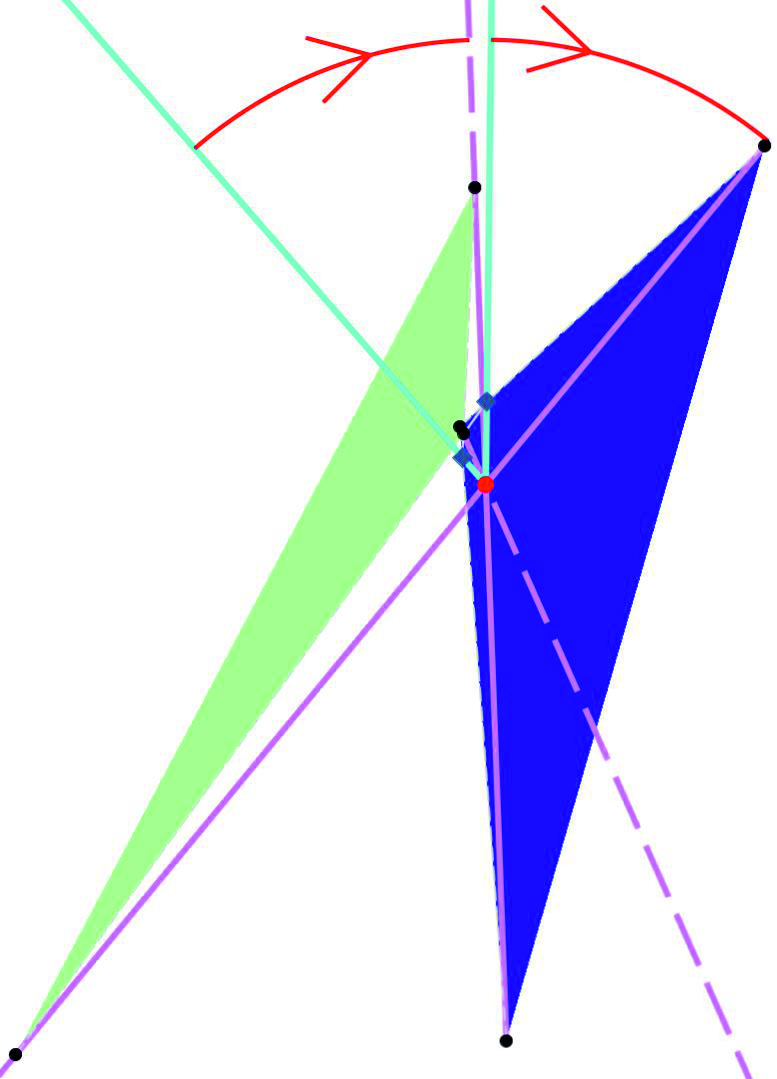}
  \begin{scriptsize}
   \put(-10,0){b)} 
   \put(38,42){$L$}
   \put(6,0){$\bVkt x_6$}
    \put(49,55){\contour{white}{$\bVkt x_1$}}
    \put(39,0){$\bVkt x_2$}
    \put(74,88){$\bVkt x_3$}
    \put(35,62){\contour{white}{$\bVkt x_4$}}

    \put(36,85){$\bVkt x_5$}
    \put(31,98){$\alpha$}
    \put(54,99){$\beta$}
  
\end{scriptsize}
\end{overpic}
\hfill
\end{center}
\caption{Stachel's test for the two solutions of Example \ref{ex:1} already illustrated in Fig.~\ref{fig:avg3rpr}a and Fig.~\ref{fig:setA}a, respectively.   
Note that, for the sake of display, the orientation of the averaged configuration
$\mathcal{G}(\bVkt X)$ has been changed. 
(a) For the first solution the angles $\alpha$ and $\beta$ become very small, therefore, their supplementary angles $\overline{\alpha}$ and $\overline{\beta}$ are displayed instead, while preserving the correct orientation. 
(b) The second solution. 
}
\label{fig:stacheltest1}
\end{figure}

\begin{figure}[htbp]
    \centering
    \begin{overpic}[width=75mm]{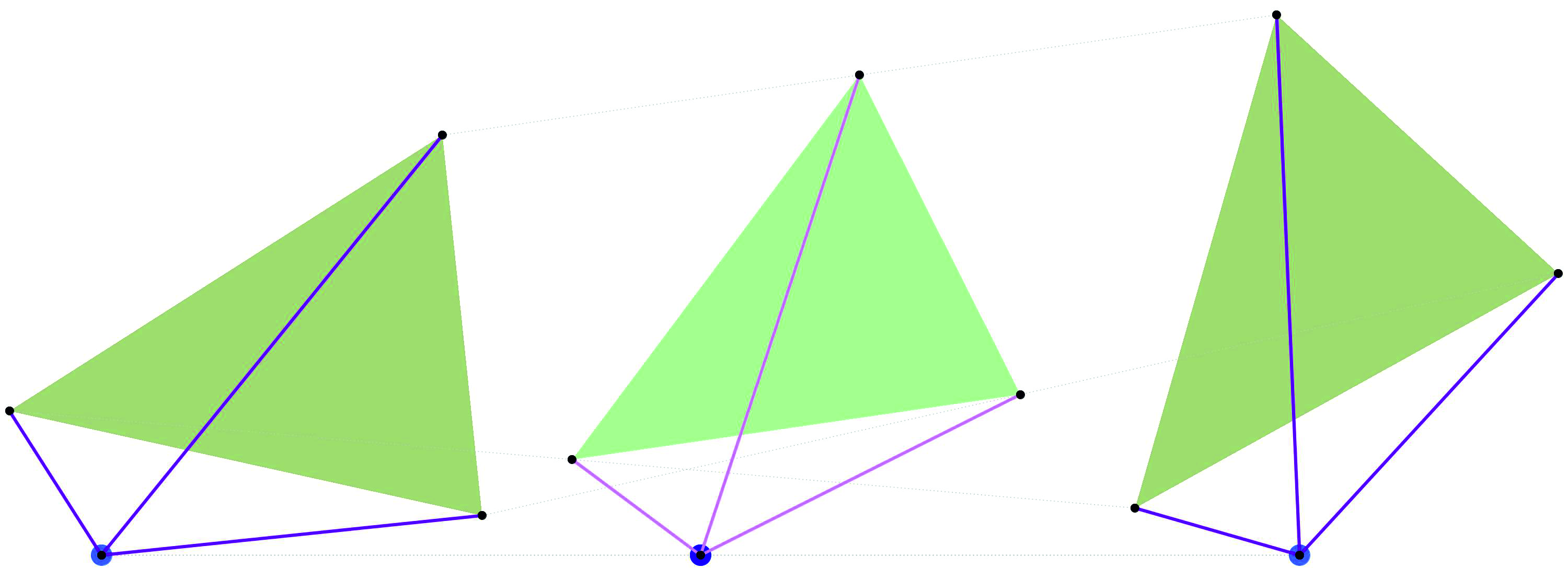}
        \scriptsize
         \put(54,33){$\bVkt x_6$} 
        \put(33,4){$\bVkt x_4$}   
        \put(65,8){$\bVkt x_5$}
        \put(34,-3){$\bVkt x_1=\bVkt x_2=\bVkt x_3$}
    \end{overpic} \medskip
\caption{$l_1=l_2=l_3=\frac{e_0}{2e_1}$, and the base of $\mathcal{G}(\bVkt X)$  is degenerated}
    \label{fig:setA-all-li-equal}
\end{figure}
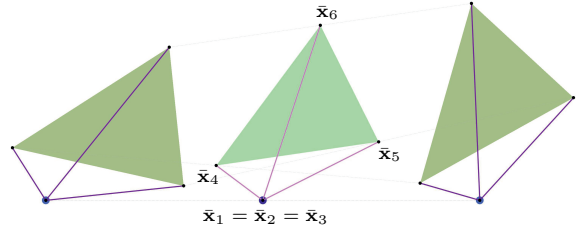

\begin{remark}
There are also parameter choices for which one or more terms $$
T_j (f_1 + 2l_k f_0)(f_1 + 2l_j f_0) 
\quad \text{for }\quad k,j \in \{1,2,3\},
$$ in Eq.\ (\ref{eq:struc})
 vanish. Such cases destroy the geometric structure of the bar-plate framework, so we exclude them.
First, suppose that 
$
e_0 - 2 e_1 l_i = 0 
$ for $i=1,2,3$.
Then necessarily
$l_1=l_2=l_3=\frac{e_0}{2e_1}.$
In this case, for both $\mathcal{G}(\Vkt X)$ and $\mathcal{G}(\Vkt X')$, the three base points coincide, so the base degenerates to a single point (see Fig.\ \ref{fig:setA-all-li-equal}). Hence the resulting object is not a valid bar-plate framework, and we discard this solution.

Beside this degenerate case it can also happen that $P_{I_i}=0$ holds true under the additional condition $f_1 + 2 l_i f_0=0$ for $i\in\left\{1,2,3\right\}$. Then the $i$-th leg has zero lengths in the averaged configuration $\mathcal{G}(\bVkt X)$, which is not admissible. This is demonstrated in Fig.\ \ref{fig:legzerof}.  \hfill $\diamond$
\end{remark}

\begin{figure}[htbp]
    \centering
    \begin{overpic}[width=75mm]{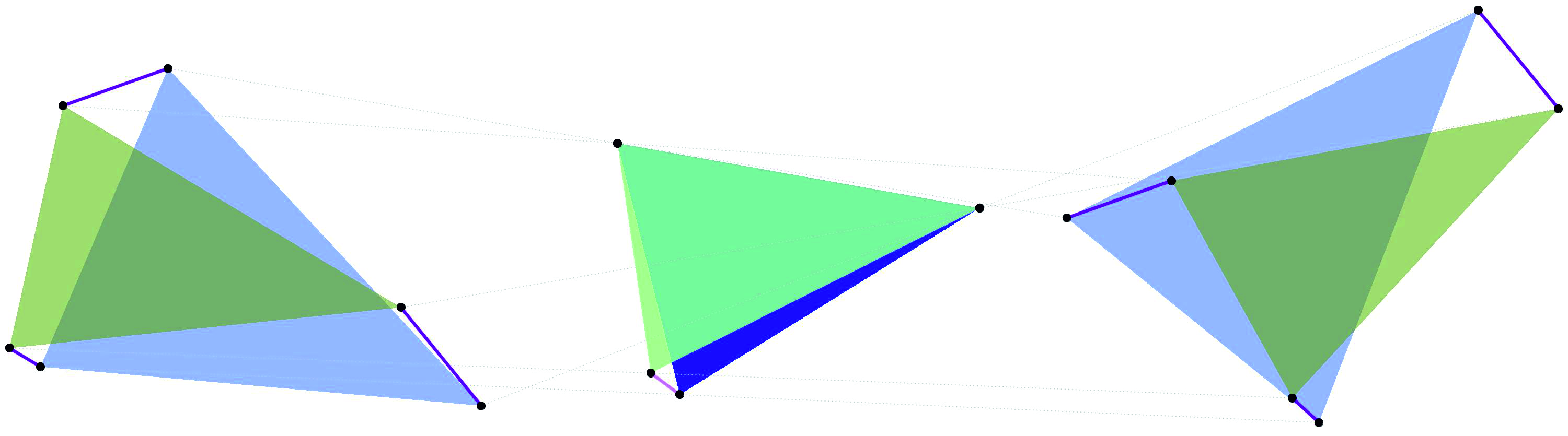}
        \scriptsize
         \put(36,3){$\bVkt x_6$} 
        \put(31,19){$\bVkt x_1=\bVkt x_4$}   
        \put(55,17){$\bVkt x_2=\bVkt x_5$}
        \put(42,-1){$\bVkt x_3$}
    \end{overpic} 
\caption{
As  $l_1=l_2=-\frac{f_1}{2f_0}$
implies $P_{I_i}=0$, the first leg and second leg  degenerate to zero length in $\mathcal{G}(\bVkt X)$.}
    \label{fig:legzerof}
\end{figure}
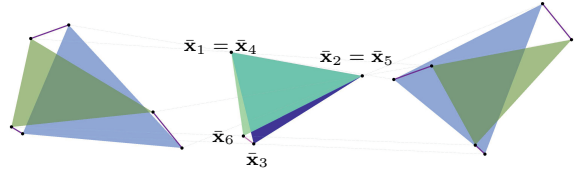
The preceding remark discusses those parameter choices that give rise to geometrically degenerate bar-plate frameworks. 
We now turn to a further algebraic source of extraneous solutions, namely, common zeros of the generators $s,s_1,\dots,s_4$ of $\mathcal{I}_2$ which do not imply $P=0$. For computation of these zeros we first replace each generator by its square-free part, namely, the product of its distinct irreducible factors, and then divide the result by the gcd of these reduced generators.

Each polynomial $q$ admits a factorisation
$$
q = \prod_k g_{k}^{m_{k}},
$$
with distinct irreducible factors $g_{k}$.
We denote by
$$
\dach{q} = \prod_k g_{k}
$$
the square-free part; i.e., the product of its distinct irreducible factors. 
Using this notation we can define the reduced polynomials 
\begin{equation}
    s^*=\frac{\dach s}{\dach P}, \qquad s_i^*=\frac{\dach s_i}{\dach P}
\end{equation}
for $i=1,\ldots, 4$, which generate the ideal
$$
\mathcal{I}^*_2 = \left\langle s^*, 
s_1^*, s_2^*, s_3^*, s_4^* \right\rangle.
$$
Recall that the ideal of singular points is given by
$$
\mathcal I_{\mathrm{sing}} := \langle s, \nabla s \rangle.
$$
Let us denote the variety of $I_{\mathrm{sing}}$ by
$V_{sing}$.
Using symbolic ideal computations in Maple (via the \texttt{PolynomialIdeals} package), we can verify the ideal containment
$$
\mathcal I_{\mathrm{sing}} \subset \mathcal I_2^*,
$$
by applying the command \texttt{IsSubset}. Then Hilbert’s Nullstellensatz implies
$$
V(\mathcal I_2^*) \subset V_{sing}.$$
Consequently, any solutions arising from the reduced generators correspond only to singular points of the variety $V_1$. In particular, no additional non-singular point of $V_1$ arise from partial intersections of the generators.

\subsubsection{Special Case:} 
$\Vkt x_{5}\neq \Vkt x'_{5}$ but $\Vkt x_{6}= \Vkt x'_{6}$ \newline
\noindent
In this case, the factorisation of $P$ yields three homogenous factors in $f_0,f_1$ given in Eq.\ (\ref{eq:2}). 
Hence, the averaged configuration $\mathcal{G}(\bVkt X)$  is at least of
flexion order $2$ only in one of the following cases.

\begin{enumerate}
\item[(a)]$f_0 e_0 + f_1 e_1 = 0$: 
If this factor vanishes, then the platform of $\mathcal{G}(\bVkt X)$ degenerates to a single point.
\item[(b)] $f_0 = 0$:
If this factor vanishes, then the base of $\mathcal{G}(\bVkt X)$ degenerates to a single point.

\item[(c)] 
$P_{I_{ii}} = 0$ and $f_0(f_0 e_0 + f_1 e_1)\neq 0$:
As $P_{I_{ii}} = 0$ is a linear homogenous equation in $f_0,f_1$ there always exists a real solution, which is illustrated by the next example.
\end{enumerate}

\begin{example}\label{ex2}
This example for the real solution of $
P_{I_{ii}}=0$ corresponds to Fig.~\ref{fig:setA}b. 
Choose
$$
a_3=2, \qquad b_3=-2, \qquad
e_0=\tfrac{3}{5}, \qquad e_1=\tfrac{4}{5},
$$
$$
l_1=10, \qquad l_2=\tfrac{1}{7}, \qquad
a_5=5, \qquad b_5=-2.
$$
Substituting these values into $P_{I_{ii}}$, we obtain
$$
P_{I_{ii}}
=
-\tfrac{3684}{175} f_0+\tfrac{39132}{175} f_1.
$$
After normalising by $f_0=1$, we obtain the real solution
$$
f_1=\tfrac{307}{3261}.
$$
For this choice of parameters, the corresponding averaged configuration
$\mathcal{G}(\bVkt X)$ also satisfies Stachel's geometric criterion  for flexion order $2$ 
(see Fig.\ \ref{fig:setAstachelcheck2}). 
\end{example}

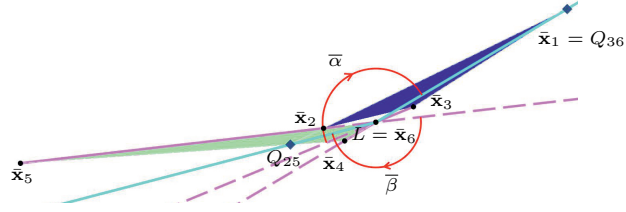
\begin{figure}[t]
    \centering
    \begin{overpic}[width=75mm]{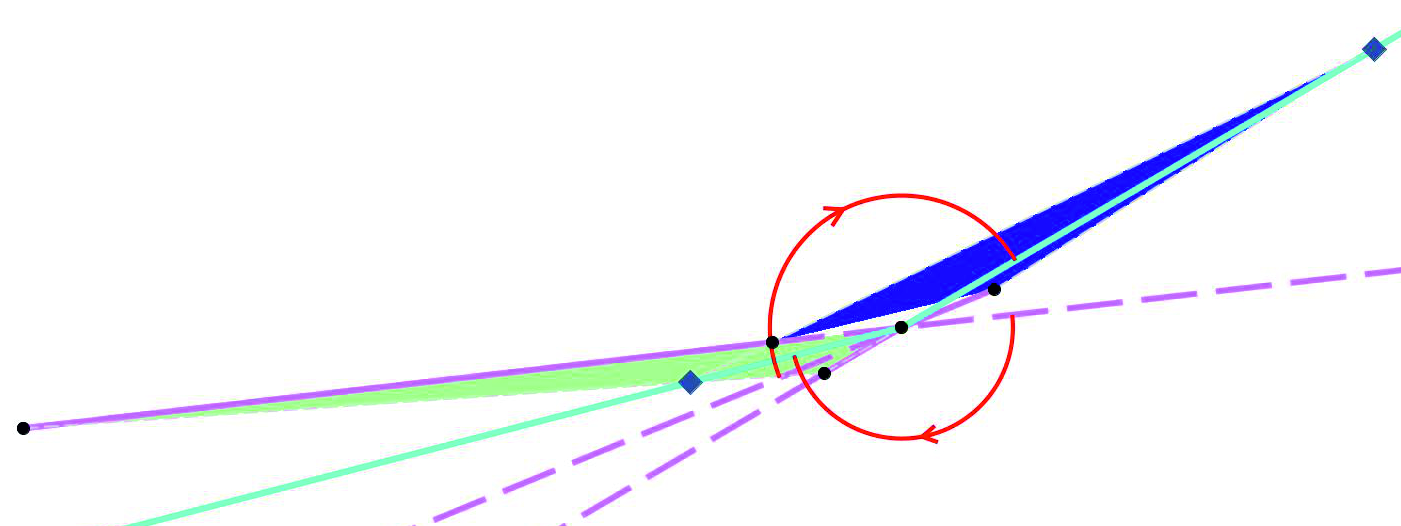}
        \scriptsize
         \put(0,4){$\bVkt x_5$} 
        \put(50,14){$\bVkt x_2$}   
        \put(55,6){$\bVkt x_4$}
        \put(93,28){$\bVkt x_1=Q_{36}$}
         \put(45,7){$Q_{25}$}
        \put(60,11){$L=\bVkt x_6$}
        \put(74,17){$\bVkt x_3$}
        \put(56,24){$\overline{\alpha}$}
        \put(66,2){$\overline{\beta}$}
    \end{overpic} 
    \caption{
    Stachel's test for the  solution of Example \ref{ex2} already illustrated in Fig.~\ref{fig:setA}b. Note that, for the sake of display, the orientation of the averaged configuration $\mathcal{G}(\bVkt X)$ has been changed. As the angles $\alpha$ and $\beta$ become very small, their supplementary angles $\overline{\alpha}$ and $\overline{\beta}$ are displayed instead.}
    \label{fig:setAstachelcheck2}
\end{figure}

\begin{remark}
As in the general case, we exclude parameter choices  $l_1=l_2=-\frac{f_1}{2f_0}$
implies $P_{I_{ii}}=0$,  that violate the geometric structure, those for which the lengths of the first leg and the second leg vanish, as these correspond to geometrically degenerate bar-plate frameworks. The details are omitted, as the procedure is identical to that of general case. 

Similarly, we apply the same algorithm as in the general case to verify that the common zeros of $s^*,s_1^*,\dots,s_4^*$ correspond to singular points of $V_1$.  \hfill $\diamond$
\end{remark}

\subsubsection{Very Special Case:} 
$\Vkt x_{5}= \Vkt x'_{5}$ and $\Vkt x_{6}= \Vkt x'_{6}$\newline
\noindent
In this case, the factorisation of $P$ yields four homogenous factors in $f_0,f_1$ given in Eq.\ (\ref{eq:3}). Thus we have to  distinguish the following cases.

\begin{enumerate}
\item[(a)]$f_0 e_0 + f_1 e_1 = 0$:
If this factor vanishes, then the platform of $\mathcal{G}(\bVkt X)$ degenerates to a single point.
\item[(b)] $f_0 = 0$:
If this factor vanishes, then the base of $\mathcal{G}(\bVkt X)$ degenerates to a single point.
\item [(c)]
$e_0 - 2 e_1 l_1=0$:
If this factor vanishes, then the platform of $\mathcal{G(\Vkt X)}$ can rotate about the point $\Vkt x_1=\Vkt x_5=\Vkt x_6$, which also holds true for the averaged framework.
\item [(d)]
$a_2 b_3 - a_3 b_2 = 0$: 
If this factor vanishes, then points $\Vkt x_2$, $\Vkt x_3$ and the origin, which equals $\Vkt x_5=\Vkt x_6$, are collinear. 
Therefore, the second leg and the third leg
become collinear in both realizations $\mathcal{G}(\Vkt X)$ and $\mathcal{G}(\Vkt X')$, which illustrated in Fig~\ref{fig:setCA}(b). 

Hence, in this case the averaged configuration has flexion order 2 independently of the chosen orientation $(f_0:f_1)$. 
\end{enumerate}

\begin{remark}
Similarly, we apply the same algorithm as in the general case to verify that the common zeros of $s^*,s_1^*,\dots,s_4^*$ correspond to singular points of $V_1$. \hfill $\diamond$
\end{remark}


\subsection{Translation}

Proceeding as in the rotational case, the polynomial $P$ admits three
non-constant factors:
$$
f_0, \qquad f_0^2 + f_1^2, \qquad P_{II},
$$
As $f_0^2 + f_1^2 \neq 0$ holds, we remain with the discussion of the following two cases. 

\begin{enumerate}
\item[(a)] $f_0 = 0$:
If this factor vanishes, then the base of $\mathcal{G}(\bVkt X)$ degenerates to a single point.

\item[(b)] $P_{II}=0$ and $f_0\neq 0$:
As $P_{II = 0}$ is a linear homogenous equation in $f_0,f_1$ there always exists a real solution, which is illustrated by the next example.
\end{enumerate}

\begin{example}\label{ex3}
This example for the real solution of $P_{II}=0$ corresponds to Fig.~\ref{fig:setA}c.
Choose
$$
a_5=6,\qquad b_5=3,\qquad a_6=-5,\qquad b_6=-2,
$$
$$
l_1=3,\qquad l_2=-2,\qquad l_3=1,\qquad d=1.
$$
Substituting these values into $P_{II}$, we obtain
$$
P_{II} = 2f_0 + 37f_1.
$$
After normalizing by $f_0=1$, this yields the real solution
$$
f_1=-\tfrac{2}{37}.
$$
In the corresponding averaged configuration
$\mathcal{G}(\bVkt X)$ all three legs are parallel, thus we have to check for Stachel's geometric criterion illustrated in Fig.\ \ref{fig:stachel}b.  This characterization holds true for the configuration $\mathcal{G}(\bVkt X)$ as illustrated in Fig.~\ref{fig:trancheck}. 
\end{example}

\begin{figure}[t]
    \centering
    \begin{overpic}[width=75mm]{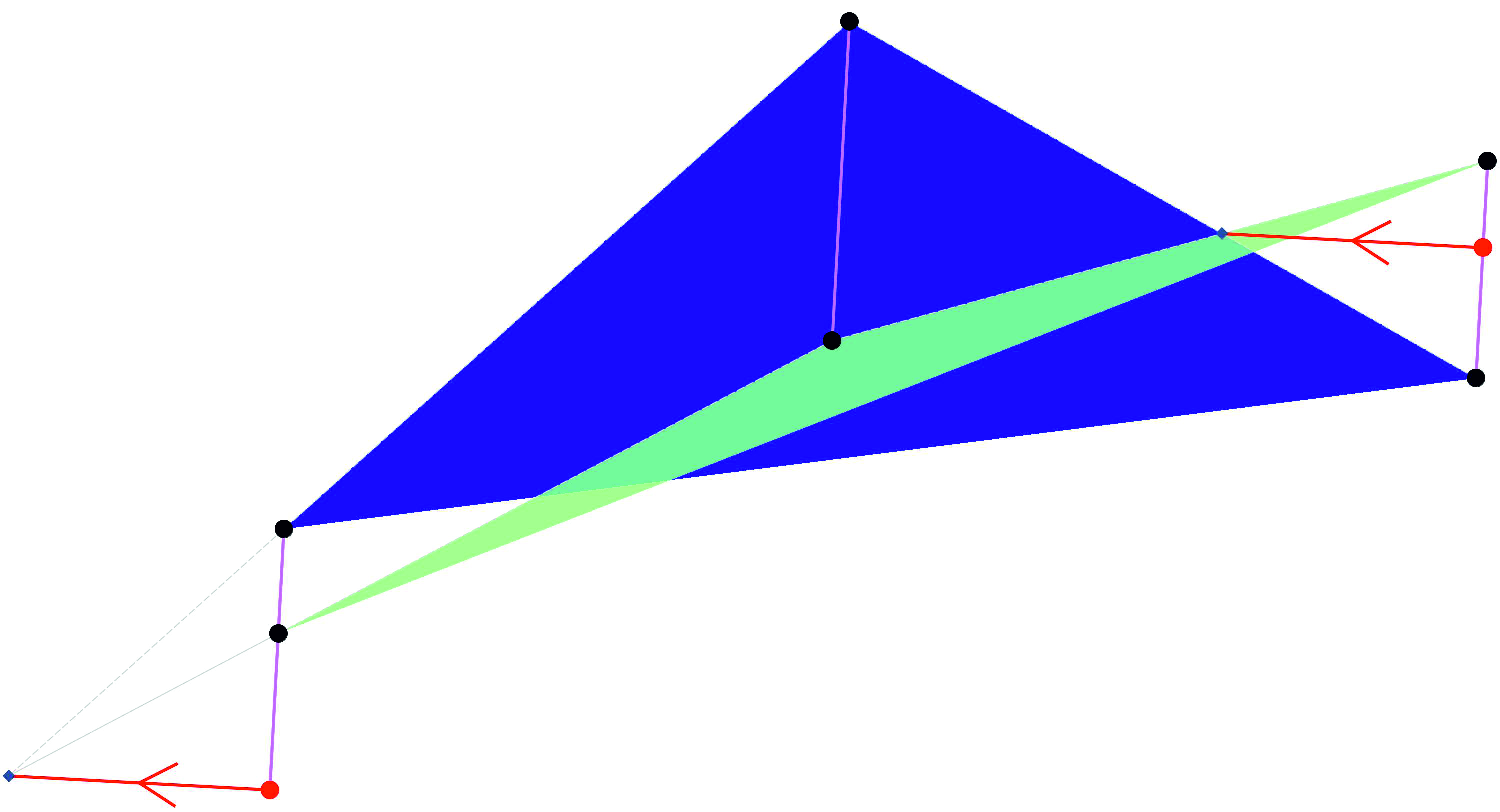}
        \scriptsize
        
         \put(-8,3){$Q_{36}$} 
         \put(79,40){$Q_{25}$}  
        \put(16,21){$\bVkt x_1$}   
        \put(96,45){$\bVkt x_5$}
        \put(95,24){$\bVkt x_2$}
        \put(19,2){$\bVkt x_4$}
        \put(50,51){$\bVkt x_3$}
        \put(54,27){\contour{white}{$\bVkt x_6$}}
    \end{overpic} 
    \caption{Stachel’s test for the solution of Example \ref{ex3} already illustrated  in Fig.~\ref{fig:setA}c.}
    \label{fig:trancheck}
\end{figure}
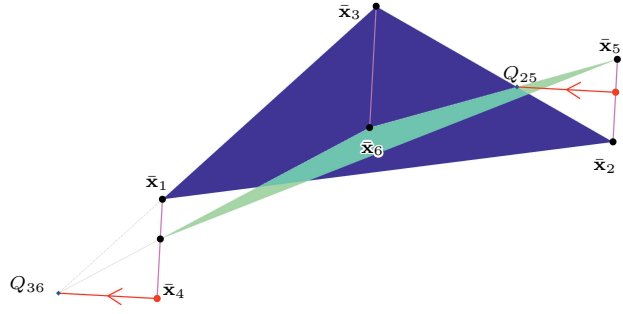

\begin{remark}\label{rem:self_mo}
Note that for $l_1 = l_2 = l_3$ both factors $(l_1-l_2)$ and $(l_1-l_3)$ vanish, and hence
$P_{II}=0$ is satisfied identically independent of the choice for $f_0,f_1$.
Geometrically, all base anchor points lie on parallel offsets of the same
perpendicular bisector direction. Consequently, the three legs of the
averaged configuration $\mathcal{G}(\bVkt X)$  are parallel, and the manipulator admits a circular translation as self-motion. 

We also apply the same algorithm as in the general case to verify that the common zeros of $s^*,s_1^*,\dots,s_4^*$ correspond to singular points of $V_1$.   \hfill $\diamond$
\end{remark}

\section{Proof of the results for Set B (Theorem \ref{thm:setB})}\label{sec:proofB} 

In Set B, the averaged configuration $\mathcal{G}(\bVkt X)$ is degenerate in the sense that the averaged platform points $\bVkt x_4,\bVkt x_5,\bVkt x_6$ are collinear and the averaged base points $\bVkt x_1,\bVkt x_2,\bVkt x_3$ are collinear. Hence $\mathcal{G}(\bVkt X)$ consists of two lines. This follows from the following

\begin{lemma}\label{lem:coll}
    Given is a triangle $\Vkt x_i, \Vkt x_j, \Vkt x_k$ and glide-reflected copy of it denoted by $\Vkt x'_i, \Vkt x'_j, \Vkt x'_k$. Then the averaged points $\bVkt x_i,\bVkt x_j,\bVkt x_k$ are collinear. 
\end{lemma}

\begin{proof}
Without loss of generality we can assume that the axis $l$ of the glide-reflection equals the $x$-axis.
Then for $\ell \in \{i,j,k\}$, the glide-reflected points $\Vkt x'_\ell$ are given by
$$
(a'_\ell,b'_\ell)^{\mathrm T} = (a_\ell + d,\,-b_\ell)^{\mathrm T},
$$
where $d$ is the translation distance along $l$. Note that for $d=0$ we have a pure reflection. As
$$
\bar{\Vkt x}_\ell
= \frac{\Vkt x_\ell + \Vkt x'_\ell}{2}
= \left(a_\ell + \frac{d}{2},\,0\right)^{\mathrm T}
$$
hold true, all points $\bar{\Vkt x}_i,\bar{\Vkt x}_j,\bar{\Vkt x}_k$ lie on the line $y=0$, and are therefore collinear.
\end{proof}

\subsection{Rotation}
Proceeding as in Set~A, we compute the determinant $s$ of the rigidity matrix of $\mathcal{G}(\bVkt X)$. Since $s \neq 0$ in general, a first-order flex is not always possible. We therefore derive the condition $s=0$ for first-order flexibility.

\subsubsection{General Case:} $\Vkt x_{5}\neq \Vkt x'_{5}$ and $\Vkt x_{6}\neq  \Vkt x'_{6}$\newline
\noindent
In this case, the determinant $s$ admits the following non-constant factors:
\begin{align*}
& e_1, \quad f_0^2 + f_1^2, \quad e_0^2 + e_1^2,\quad
& e_0 f_1 + e_1 f_0, \quad \Psi_5, \quad \Psi_6,\qquad
& f_0 A_{I_{i}} + f_1 B_{I_{i}},
\end{align*}
with $A_{I_{i}}$ and $B_{I_{i}}$ from Theorem \ref{thm:setB} and
\begin{align*}
\Psi_5 &= a_5(e_0 f_0 - e_1 f_1) - b_5(e_0 f_1 + e_1 f_0),\\
\Psi_6 &= a_6(e_0 f_0 - e_1 f_1) - b_6(e_0 f_1 + e_1 f_0).
\end{align*}
Since we have $(f_0,f_1)\neq (0,0)$, $(e_0,e_1)\neq (0,0)$ and $e_1\neq 0$ we remain with the discussion of the following four cases.

\begin{enumerate}
\item[(a)] $e_0 f_1 + e_1 f_0 = 0$: 
In this case, the first leg of $\mathcal{G}(\bVkt X)$ degenerates to zero length, which violates the bar-plate framework assumptions.

\item[(b)]
$\Psi_5=0$: 
Then the second leg of $\mathcal{G}(\bVkt X)$ degenerates to zero length, which violates the bar-plate framework assumptions.
\item[(c)]$
\Psi_6=0$: 
Then the third leg of $\mathcal{G}(\bVkt X)$ degenerates to zero length, which violates the bar plate framework assumptions.

\item[(d)] 
$f_0 A_{I_{i}} + f_1 B_{I_{i}}=0$ and $(e_0 f_1 + e_1 f_0)\Psi_5\Psi_6\neq 0$:
In this case, the realization corresponds to a genuine first-order flex, which is illustrated by the following example.
\end{enumerate}

\begin{example}
  This example for a real solution of $f_0 A_{I_{i}} + f_1 B_{I_{i}}=0$ corresponds to Fig.~\ref{fig:setB}b.
Choose
$$
a_5 = 6, \qquad b_5 = -4, \qquad
a_6 = 7, \qquad b_6 = 0,
$$
$$
e_0 = \tfrac{8}{17}, \qquad e_1 = \tfrac{15}{17}, \qquad
l_1 = 2, \qquad l_2 = 3, \qquad l_3 = 1.
$$
Substituting these values into $A_{I_{i}}$ and $B_{I_{i}}$, we obtain
$$
A_{I_{i}} = -\tfrac{9668}{289}, \qquad
B_{I_{i}} = \tfrac{7290}{289}.
$$
Using the normalized homogeneous choice $f_0 = 1$ yields
$$
(f_0, f_1) = \left(1, \tfrac{4834}{3645}\right).
$$
\end{example}
\subsubsection{Special Case:} $\Vkt x_{5}\neq \Vkt x'_{5}$ and $\Vkt x_{6}=  \Vkt x'_{6}$
\newline
\noindent
In this case, the determinant $s$ admits the following non-constant factors:
\begin{align*}
& e_1,\quad f_0^2+f_1^2,\quad e_0^2+e_1^2,\quad e_0f_1+e_1f_0,\quad \Psi_3,\quad \Psi_5,\quad f_0A_{I_{ii}}+f_1B_{I_{ii}},
\end{align*}
with $A_{I_{ii}}$ and $B_{I_{ii}}$ from Theorem \ref{thm:setB} and $\Psi_3 = a_3f_0-b_3f_1$.

Since we have $(f_0,f_1)\neq (0,0)$, $(e_0,e_1)\neq (0,0)$ and $e_1\neq 0$ we remain with the discussion of the following four cases.

\begin{enumerate}
\item[(a)]$e_0f_1+e_1f_0=0$:
In this case, the first leg degenerates to zero length, and the realization is degenerate.

\item[(b)] $\Psi_5=0$: 
Then the second leg collapses to zero length, and the realization is degenerate.

\item[(c)]$\Psi_3=0$: 
Then the third leg has zero length, and the realization is degenerate.
\item[(d)] 
$f_0 A_{I_{ii}} + f_1 B_{I_{ii}}=0$ and $(e_0 f_1 + e_1 f_0)\Psi_3\Psi_5\neq 0$: 
In this case, the realization corresponds to a genuine first-order flex, which is illustrated by the following example.
\end{enumerate}

\begin{example}
This example for a real solution of $f_0 A_{I_{ii}} + f_1 B_{I_{ii}}=0$ corresponds to  Fig.~\ref{fig:setCA}c.
Using the same parameter choice
$$
a_5=3,\qquad b_5=-1,\qquad
e_0=\tfrac35,\qquad e_1=\tfrac45,\qquad
l_1=2,\qquad l_2=-1,
$$
we obtain
$$
91f_0-138f_1=0.
$$
Using the projective normalization $f_0=1$, we get
$$
(f_0,f_1)=\left(1,\tfrac{91}{138}\right).
$$
\end{example}

\subsection{Translation}

In the translational case the determinant $s$ admits the following non-constant factors:
\begin{align*}
& f_1,\qquad f_0^2+f_1^2,\qquad f_0A_{II}+f_1B_{II}
\end{align*}
with $A_{II}$ and $B_{II}$ from Theorem \ref{thm:setB}. 
Since $(f_0,f_1)\neq (0,0)$, 
we distinguish the following cases.
\begin{enumerate}

\item[(a)] $f_1 = 0$:
If this factor vanishes, then the corresponding realisation degenerates, as all six points of $\mathcal{G}(\bVkt X  )$ collapse into a single point. 

\item[(b)] 
$f_0 A_{II} + f_1 B_{II} = 0$ and $f_1 \neq 0$:
In this case, the realization corresponds to a genuine first-order flex, which is illustrated in the following example. 
Moreover, also the first paragraph of Remark \ref{rem:self_mo} holds true for this case.
\end{enumerate}

\begin{example}
This example for a real solution of $f_0 A_{II} + f_1 B_{II} = 0$ corresponds to Fig.~\ref{fig:setB}c.
Choose
$$
a_5=3,\quad b_5=6,\quad
a_6=-3,\quad b_6=4,
\quad
l_1=2,\quad l_2=-3,\quad l_3=4.
$$
Substituting these values into the equation yields
$$
9f_0+25f_1=0.
$$
Imposing the normalization $f_0 = 1$, we obtain
\[
(f_0,f_1) = \left(1, -\tfrac{9}{25}\right).
\]
\end{example}

\section{Proof of the results for Set C (Theorem \ref{thm:setC})}\label{sec:proofC}

\subsection{Glide-reflection}
In the case of a glide-reflection, it is important to notice that the averaged platform $\mathcal{G(\bVkt X)}$ will always degenerate into a line due to 
Lemma \ref{lem:coll}.

In this case, the determinant $s$ admits the following non-constant factors:
\begin{align*}
& f_0,\quad d,\quad
2f_0l_1+f_1,\quad
2f_0l_2+f_1,\quad
2f_0l_3+f_1,\quad
& a_5b_6-a_6b_5-a_5+a_6.
\end{align*}

Since we assume $d\neq 0$ in this glide-reflection case, we have to distinguish between the following cases. 
\begin{enumerate}[(a)]
    \item 
    $f_0 = 0$: 
    In this case, the platform of $\mathcal{G}(\bVkt X  )$ degenerates.
    \item 
    $2f_0 l_i + f_1 = 0$ for some $i \in \{1,2,3\}$:
    In this situation, the length of the corresponding $i$-th leg of $\mathcal{G}(\bVkt X  )$ becomes zero, which violates the geometric constraints of the bar-plate framework.
    
    \item 
    $a_5 b_6 - a_6 b_5 - a_5 + a_6 = 0$: This condition implies the collinearity of the three platform points of $\mathcal{G}(\Vkt X  )$ and therefore it contradicts our assumptions. 
\end{enumerate}

\subsection{Reflection}

\subsubsection{General Case:} $\Vkt x_{5}\neq \Vkt x'_{5}$ and $\Vkt x_{6}\neq \Vkt x'_{6}$
\newline
\noindent
In Case ($II_i$), all six points are collinear. Under the reflection
$(a,b)\mapsto(a,-b)$, the averaged platform points satisfy
$$
\Vkt x_j = (a_j,0),
$$
and hence lie on the $x$-axis.
Moreover, for each leg, the midpoint $m_i$ of $\Vkt x_{i+3}$ and $\Vkt x'_{i+3}$ is given by
$$
\Vkt m_i = (a_{i+3},0),
$$
and the corresponding bisector direction is parallel to $(1,0)$. Therefore, all base anchor points also lie on the $x$-axis. Therefore, all six points are collinear which corresponds to a singular point of $V_1$.

\subsubsection{Special Case:} $\Vkt x_{5}\neq \Vkt x'_{5}$ but $\Vkt x_{6}= \Vkt x'_{6}$ \newline
\noindent
In this case, the first-order flexion condition $s = 0$ is always satisfied. 
Therefore, we investigate the higher-order condition by computing the greatest common divisor $P$, which factors into
$$
f_0, \qquad b_3, \qquad P_{II_{ii}},
$$ 
with $P_{II_{ii}}$ from Theorem \ref{thm:setC}. In the following we discuss these three cases.
\begin{enumerate}
\item[(a)] $f_0 = 0$:
    In this case, the platform of $\mathcal{G}(\bVkt X  )$ degenerates to a point.

\item[(b)] $b_3=0$:
In this case, $\Vkt x_3$ lies on the reflection axis. Consequently, all three averaged platform points and all three averaged base points lie on the $x$-axis. Therefore the averaged configuration corresponds to a singular point of $V_1$.

\item[(c)] $P_{II_{ii}}=0$ and $f_0b_3\neq 0$:
As the polynomial $P_{II_{ii}}$ is quadratic in $(f_0, f_1)$ there exist parameter choices such that two distinct real solutions exist, as demonstrated by the following example.
 
\end{enumerate}

\begin{example}\label{exl}
This example for two real solutions of $P_{II_{ii}}=0$ correspond to Fig.~\ref{fig:setCA}a and Fig.~\ref{fig:setcIii2}.
Choose
$$
a_3 = 7,\qquad b_3 = -2,\qquad
a_5 = 4,\qquad b_5 = 5, \qquad a_6 = 9,
$$
$$
b_6 = 0,\qquad
e_0 = \tfrac35,\qquad e_1 = \tfrac45, \qquad
l_1 = 10,\qquad l_2 = -9,\qquad l_3 = 2.
$$

The corresponding real solutions for $f_0$ and $f_1$ are
$$(f_0,f_1)=\left(1,\,
-\tfrac{1}{149790}
\sqrt{11226910290-23666820\sqrt{221549}}\right)
$$
$$(f_0,f_1)=\left(1,\,
-\tfrac{1}{149790}
\sqrt{11226910290+23666820\sqrt{221549}}\right)$$
For both solutions, the five points
$
\bVkt x_1,\ \bVkt x_2,\ \bVkt x_4,\ \bVkt x_5,\ \bVkt x_6
$
are collinear, while $\bVkt x_3$ is the exceptional point.
Then Stachel's criterion for second-order flexion is satisfied since $\alpha=\beta=0$.
\end{example}

\begin{figure}[t]
\begin{center}

\hfill
\begin{overpic}[height=45mm]{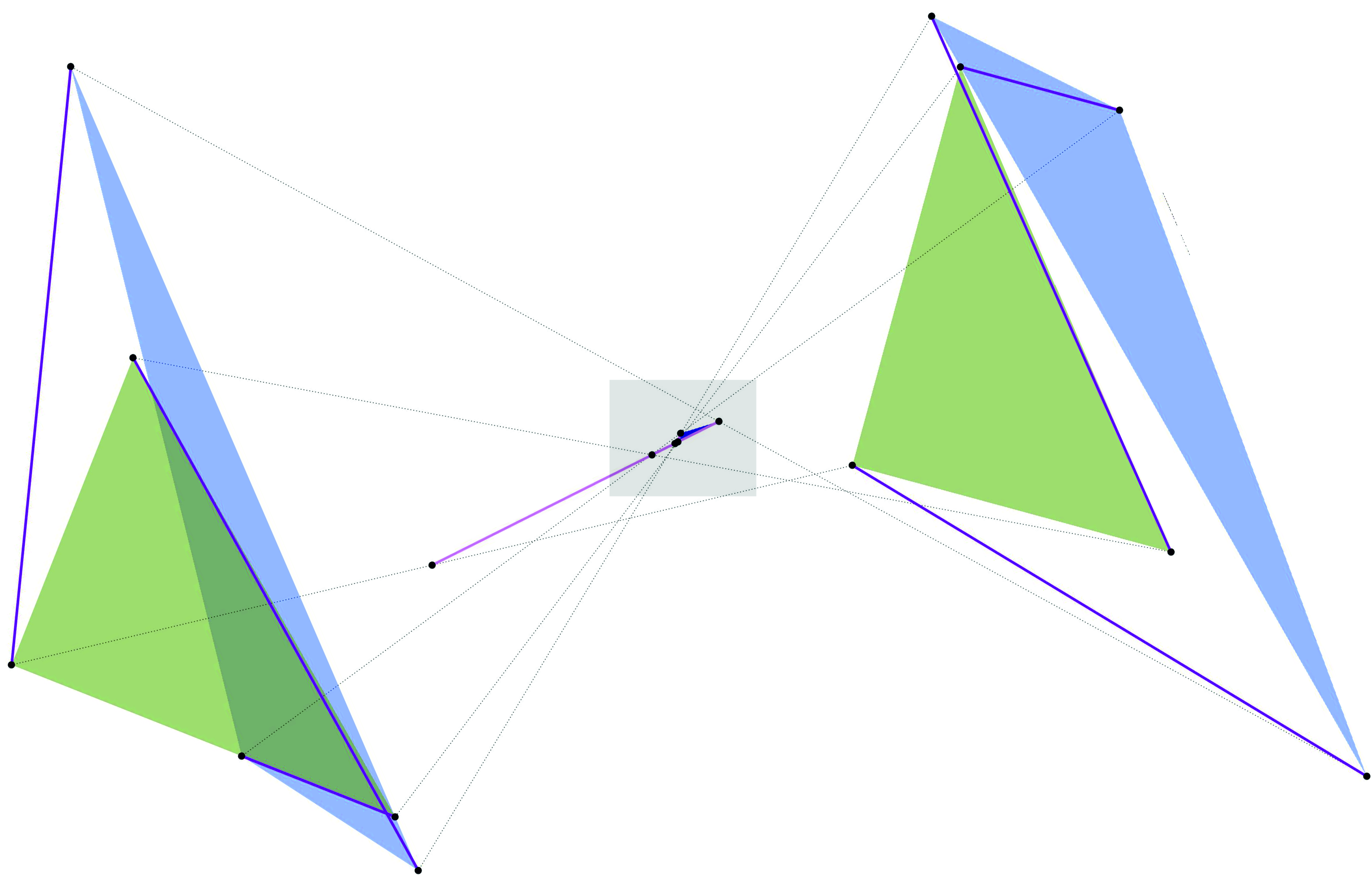}
  \scriptsize
    \put(-3,0){(a)}
  \put(28,19){$\bVkt x_6$}
\end{overpic}
\hfill
\raisebox{3mm}{ 
\begin{overpic}[height=28mm]{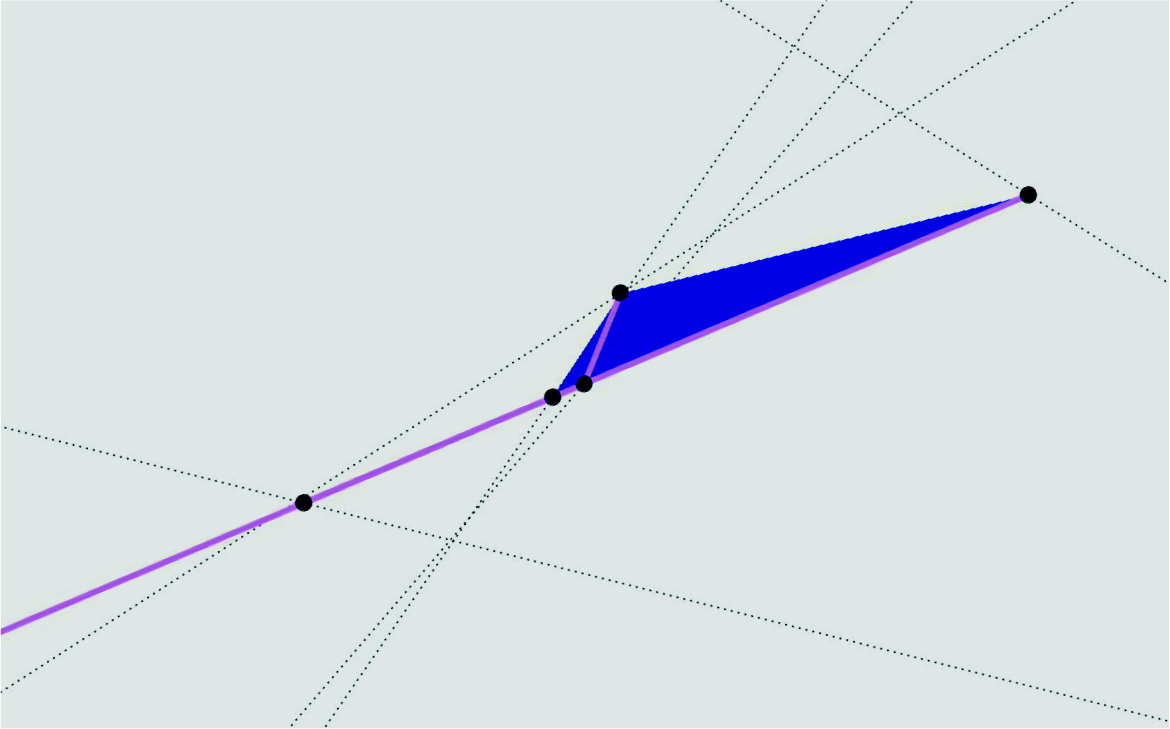}
    \scriptsize
    \put(0,-6){(b)}
    
    \put(22,14){$\bVkt x_5$}
    \put(85,40){$\bVkt x_2$}
    \put(51,42){$\bVkt x_3$}
    \put(50,25){$\bVkt x_4$}
    \put(43,23){$\bVkt x_1$}
\end{overpic}
}
\hfill
\end{center}
\caption{(a)
Illustration of second solution of Example \ref{exl}, while the first one is already displayed  in  Fig.~\ref{fig:setCA}a.  
(b) Zoomed view of the central region.}
\label{fig:setcIii2}
\end{figure}
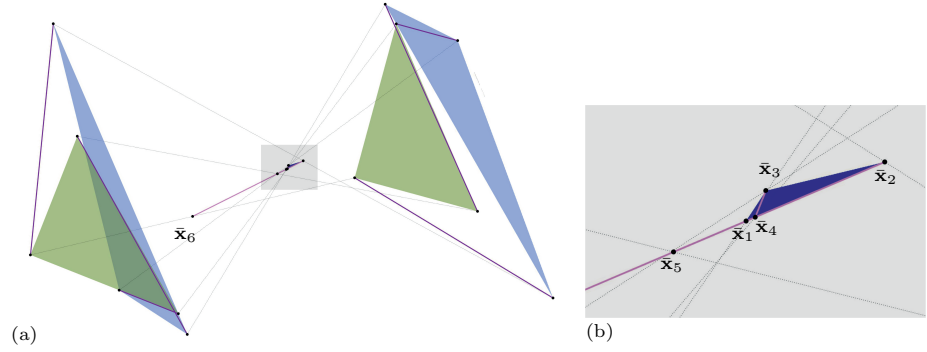

\begin{remark}
As in last section, degenerate parameter choices (e.g.\ zero leg lengths) are excluded. The analysis of the $V(\mathcal I_2^*)$ follows the same procedure as earlier section, and the corresponding common zeros are again contained in the singularity variety $V_{sing}$. \hfill $\diamond$
\end{remark}

\subsubsection{Very Special Case:}
$\Vkt x_{5}= \Vkt x'_{5}$ and $\Vkt x_{6}= \Vkt x'_{6}$ \newline
\noindent
In this case, the first-order flexion condition $s=0$ is not identically satisfied. As $s$ splits up into the follow factors
$$
b_2,\qquad b_3,\qquad a_5 - a_6,\qquad 2f_0 l_1 + f_1,
$$
we have to distinguish the following cases.

\begin{enumerate}
\item[(a)] $b_2 = 0$:
In this case, the five points of the averaged configuration $\mathcal{G}(\bVkt X )$ lie on a common line, with only $\bVkt x_6$ off the line. Consequently, leg 1 and leg 2 have the same direction.

\item[(b)]$b_3 = 0$: 
Now the five points of the averaged configuration $\mathcal{G}(\bVkt X )$ lie on a common line, with only $\bVkt x_5$ off the line.
Consequently, leg 1 and leg 3 have the same direction.

\item[(c)] 
$a_5 = a_6$: In this case the corresponding two platform points of $\mathcal{G}(\Vkt X )$ coincide. In the averaged configuration $\mathcal{G}(\bVkt X)$ leg 1 becomes collinear with the platform.

\item[(d)] $2 f_0 l_1 + f_1 = 0$: Finally,  the first leg of $\mathcal{G}(\bVkt X )$ become zero length.
This closes the discussion of all cases. 
\end{enumerate}

\end{document}